\documentclass[11pt,letterpaper,english]{article}

\usepackage[T1]{fontenc}
\usepackage[utf8]{inputenc}
\usepackage[english]{babel}
\usepackage[margin=1in]{geometry}
\usepackage[numbers]{natbib}
\setcitestyle{nocompress}
\bibpunct{(}{)}{;}{a}{}{,}

\usepackage{graphicx}
\usepackage[cmex10]{amsmath}
\usepackage{authblk}
\usepackage{macros}

\begin{document}
\title{Cumulative Prospect Theory Meets Reinforcement Learning: Prediction and Control}

\author[1]{Prashanth L.A.\thanks{prashla@isr.umd.edu}}
\author[2]{Cheng Jie\thanks{cjie@math.umd.edu}}
\author[3]{Michael Fu\thanks{mfu@isr.umd.edu}}
\author[4]{Steve Marcus\thanks{marcus@umd.edu}}
\author[5]{Csaba Szepesv\'ari\thanks{szepesva@cs.ualberta.ca}}
\affil[1]{\small Institute for Systems Research, University of Maryland}
\affil[2]{\small Department of Mathematics, University of Maryland}
\affil[3]{\small Robert H. Smith School of Business \& Institute for Systems Research,
University of Maryland}
\affil[4]{\small Department of Electrical and Computer Engineering \& Institute for Systems Research,
University of Maryland}
\affil[5]{\small Department of Computing Science,
University of Alberta}

\renewcommand\Authands{ and }

\date{}

\maketitle

\begin{abstract}
Cumulative prospect theory (CPT) is known to model human decisions well, with substantial empirical evidence supporting this claim. 
CPT works by distorting probabilities and is more general than the classic expected utility and coherent risk measures. We bring this idea to a risk-sensitive reinforcement learning (RL) setting and design algorithms for both estimation and control.
The RL setting presents two particular challenges when CPT is applied: estimating the CPT objective requires estimations of the {\it entire distribution} of the value function and finding a {\it randomized} optimal policy.
The estimation scheme that we propose uses the empirical distribution to estimate the CPT-value of a random variable. We then use this scheme in the inner loop of a CPT-value optimization procedure that is based on the well-known simulation optimization idea of simultaneous perturbation stochastic approximation (SPSA).
We provide theoretical convergence guarantees for all the proposed algorithms and also 
illustrate the usefulness of CPT-based criteria in a traffic signal control application.
\end{abstract}


\section{Introduction}
\label{sec:introduction}

Since the beginning of its history, mankind has been deeply immersed in 
	designing and improving systems to serve humans needs.
Policy makers are busy with designing 
	systems that serve the education, transportation, economic, health and other 
	needs of the public,
while private sector enterprises or hard at creating 
	and optimizing systems to serve further 
	more specialized needs of their customers.
While it has been long recognized that 
	understanding human behavior is a prerequisite 
	to best serving human needs \citep[e.g.,]{Simon:1959kd},
	it is only recently that this approach is gaining a wider recognition.%
\footnote{
As evidence for this wider recognition in the public sector,
we can mention a recent executive order of the White House
calling for the use of behavioral science in public policy making, 
or the establishment of the ``Committee on Traveler Behavior and Values'' in the Transportation
Research Board in the US.}

In this paper we consider \emph{human-centered reinforcement learning problems}
where the  reinforcement learning agent controls a system 
to produce long term outcomes (``return'') that are maximally aligned with the preferences of 
one or possibly multiple humans, an arrangement shown on Figure~\ref{fig:flow}.
As a running example, consider traffic optimization where the goal is to maximize
travelers' satisfaction, a challenging problem 
in big cities.
In this example, the outcomes (``return'') are travel times, or delays. 
To capture human preferences, the outcomes are mapped to a single numerical quantity.
While preferences of rational agents facing uncertain situations can be modeled using expected utilities 
(i.e., the expectation of a nonlinear transformation, such as the exponential function, of the rewards or costs) 
\citep{NeuMo44,fishburn1970expectedutility}, 
it is well known that
\todoc{Some use the world uncertainty in relation to situations
which can not be modeled probabilistically, and use risk in relation to situations that can be modeled probabilistically.
So we are inconsistent with these people.}
	humans are subject to various emotional and cognitive biases,
	and, the psychology literature agrees that human preferences 
	are inconsistent with expected utilities regardless of what nonlinearities are used
	 \citep{allais53,ellsberg61,kahneman1979prospect}.
An approach that gained 
	strong support amongst psychologists, behavioral scientists and economists  \citep[e.g.,][]{starmer2000developments,quiggin2012generalized}
	is based on \cite{kahneman1979prospect}'s celebrated \emph{prospect theory} (PT).
Therefore, in this work, we will base our models of human preferences on this theory.
More precisely, we will use \emph{cumulative prospect theory} (CPT),
 	a later, refined variant of prospect theory due to \cite{tversky1992advances}, 
	which is even more empirically and theoretically supported than prospect theory \citep[e.g.,][]{Barberis:2012vs}.
CPT generalizes expected utility theory in that in addition to having a utility function transforming
	the outcomes, another function is introduced which distorts the probabilities in the cumulative distribution function.
As compared to prospect theory, CPT is monotone with respect to stochastic dominance, a property
	that is thought to be useful and (mostly) consistent with human preferences\footnote{See Appendix \ref{sec:appendix-cpt-intro} for an introduction to PT/CPT and a description of the Allais paradox.}.	

\begin{figure}[h]
\centering
\tabl{c}{
\includegraphics[width=3.8in]{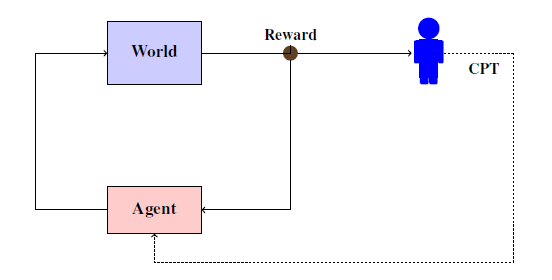}}
%
\caption{Operational flow of a human-based decision making system
}
\label{fig:flow}
\end{figure}
\paragraph{Our contributions:}

To our best knowledge, we are the first to investigate (and define) human-centered RL, and, in particular, 
this is the first work to combine CPT with RL. Although on the surface the combination may seem straightforward, in fact there are many research challenges that arise from trying to apply a CPT objective in the RL framework, as we will soon see. 
We outline these challenges as well as our solution approach below. 

The first challenge stems from the fact that the CPT-value assigned to a random variable is defined through a nonlinear transformation of certain cumulative distribution functions associated with the random variable (cf. \cref{sec:cpt-val} for the definition). 
Hence, even the problem of estimating the CPT-value given a random sample requires some effort.
In this paper, we consider a natural quantile-based estimator and analyze its behavior.
Under certain technical assumptions, we prove consistency and sample complexity bounds, the latter based on the
 Dvoretzky-Kiefer-Wolfowitz (DKW) theorem.
As an example, we show that the sample complexity for estimating the CPT-value 
for Lipschitz probability distortion (so-called ``weight'') functions is  $O\left(\frac1{\epsilon^2}\right)$, which coincides with the canonical rate for Monte Carlo-type schemes. Since weight-functions that fit well to human preferences are only  \holder continuous, we also consider this case and find that (unsurprisingly) the sample complexity  jumps to $O\left(\frac1{\epsilon^{2/\alpha}}\right)$ where $\alpha\in (0,1]$ is the weight function's \holder exponent.

The work on estimating CPT-values forms the basis of the algorithms that we propose to maximize CPT-values based on interacting either with a real environment, or a simulator. We set up this problem as an instance of policy search: We consider smoothly parameterized policies whose parameters are tuned via stochastic gradient ascent. For estimating gradients, we use two-point randomized gradient estimators, borrowed from simultaneous perturbation stochastic approximation (SPSA), a widely used algorithm in \textit{simulation optimization} \cite{fu2015handbook}.
Here a new challenge arises which is that we can only feed the two-point randomized gradient estimator with \emph{biased} estimates of the CPT-value. To guarantee convergence, we propose a particular way of controlling the arising bias-variance tradeoff.
%
%

To put things in context, risk-sensitive reinforcement learning problems are generally hard to solve. 
For a discounted MDP, \citet{Sobel82VD} showed that there exists a Bellman equation for the variance of the return, but the underlying Bellman operator is not necessarily monotone and this rules out policy iteration as a solution approach for variance-constrained MDPs.
Further, even if the transition dynamics are known, \citet{mannor2013algorithmic} show that finding a globally mean-variance optimal policy in a discounted MDP is NP-hard.
For average reward MDPs, \citet{filar1989variance} motivate a different notion of variance and then provide NP-hardness results for finding a globally variance-optimal policy.
CVaR as a risk measure is equally complicated as the measure here is a conditional expectation, where the conditioning is on a low probability event. Apart from the hardness of finding CVaR-optimal solutions, estimating CVaR for a fixed policy in a typical RL setting itself is a challenge considering CVaR relates to rare events and to the best of our knowledge, there is no algorithm with theoretical guarantees to estimate CVaR without wasting a lot of samples. There are proposals based on importance sampling (cf. \citealt{prashanth2014policy,tamar2014optimizing}), but they lack theoretical guarantees. \todoc{I do not know the simulation optimization literature very well, but I think a major topic is to estimate CVaR and other tail measures and there are many methods with guarantees. Books are written about this. For example, Kroese et al., or Gerardo Rubino, Bruno Tuffin:	Rare event simulation using Monte Carlo methods.}

We derive a \textit{provably} sample-efficient \todoc{lower bound? then mention it earlier} scheme for estimating the CPT-value (see next section for a precise definition) for a given policy and use this as the inner loop in a policy optimization scheme. Finally, we point out that the CPT-value that we define is a generalization of the above previous works in the sense that one can recover the regular value function and the risk measures such as VaR and CVaR by appropriate choices of a the distortions used in the definition of the CPT value.

The work closest to ours is by \cite{lin2013stochastic}, who proposes a CPT-measure for an abstract MDP setting.
 We differ from \cite{lin2013stochastic} in several ways:
\begin{inparaenum}[\it (i)]
\item We do not assume a nested structure for the CPT-value 
and this implies the lack of a Bellman equation for our CPT measure;
\item we do not assume model information, i.e., we operate in a model-free RL setting. Moreover, we develop both estimation and control algorithms with convergence guarantees for the CPT-value function.
\end{inparaenum}

The rest of the paper is organized as follows: 
In Section~\ref{sec:cpt-val}, we introduce the notion of CPT-value of a random variable $X$.
In Section~\ref{sec:cpt-sampling}, we
describe a quantile-based scheme for estimating the CPT-value. In Section \ref{sec:1spsa}, we present a gradient-based algorithm for optimizing the CPT-value. We present the simulation results for a traffic signal control application in Section~\ref{sec:expts} and finally, provide the concluding remarks in Section~\ref{sec:conclusions}.
Appendix \ref{sec:appendix-cpt-intro} provides background material for CPT and Appendix \ref{sec:cpt-ssp} makes a special case of the CPT-value in a stochastic shortest path problem.  We provide the proofs of convergence for all the proposed algorithms in Appendices \ref{appendix:cpt-est}--\ref{appendix:1spsa}. Further, Appendix \ref{sec:2spsa} describes a second-order algorithm for CPT-value optimization.
\section{CPT-value}
\label{sec:cpt-val}

For a real-valued random variable $X$, we introduce a ``CPT-functional'' that replaces the traditional expectation operator. 
\todoc{Can't we just remove that section and instead simply explain that whatever way you get returns given a policy, you can apply the CPT-functional to the returns? In particular, state observability, inherent in MDP formulations, is totally unnecessary. As is to assume finite state or action spaces.}
The CPT-value of the random variable $X$ is defined as
\begin{align}
\C_{u,w}(X) = &\intinfinity w^+(P(u^+(X)>z)) dz
& - \intinfinity w^-(P(u^-(X)>z)) dz, \label{eq:cpt-general}
\end{align}
where $u=(u^+,u^-)$, $w=(w^+,w^-)$, $u^+,u^-:\R\rightarrow \R_+$ and $w^+,w^-:[0,1] \rightarrow [0,1]$ are continuous (see assumptions (A1)-(A2) in Section \ref{sec:cpt-sampling} for precise requirements on $u$ and $w$). For notational convenience, since $u,w$ will be fixed, we drop the dependence on $u,w$ and use $\C(X)$ to denote the CPT-value.
\cref{fig:u} shows an example of the utility functions $u= (u^+,u^-)$ and how they relate to each other, while \cref{fig:w} shows an example of a typical weight function.

 \begin{figure}[t]
   \centering
\tabl{c}{
\includegraphics[width=3.8in]{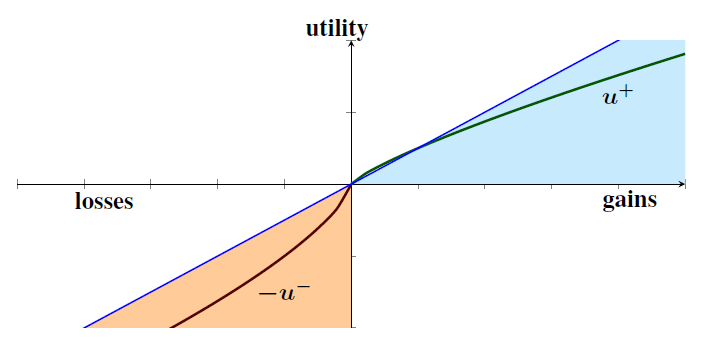}}
\caption{An example of a utility function.}
\label{fig:u}
\end{figure}

In the definition, $u^+, u^-$ are utility functions corresponding to gains ($X \ge 0$) and losses ($X \le 0$), respectively. For example, consider a scenario where one can either earn \$$500$ w.p $1$ or earn \$$1000$ w.p. $0.5$ (and nothing otherwise). The human tendency is to choose the former option of a certain gain. If we flip the situation, i.e., a certain loss of \$$500$ or a loss of \$$1000$ w.p. $0.5$, then humans choose the latter option.  Handling losses and gains separately is a salient feature of CPT, and this addresses the tendency of humans to play safe with gains and take risks with losses - see Fig \ref{fig:u}.  In contrast, the traditional value function makes no such distinction between gains and losses.  

The functions $w^+, w^-$, called the weight functions, capture the idea that humans deflate high-probabilities and inflate low-probabilities.
For example, humans usually choose a stock that gives a large reward, e.g., one million dollars w.p. $1/10^6$ over one that gives \$$1$ w.p. $1$ and the reverse when signs are flipped. 
Thus the value seen by the human subject is non-linear in the underlying probabilities -- an observation backed by strong empirical evidence \citep{tversky1992advances,Barberis:2012vs}.  
In contrast,the traditional value function is linear in the underlying probabilities. 
As illustrated with $w=w^+=w^-$ in Fig \ref{fig:w}, the weight functions are continuous, non-decreasing and  have the range $[0,1]$ with $w^+(0)=w^-(0)=0$ and $w^+(1)=w^-(1)=1$. 
\citet{tversky1992advances} recommend $w(p) = \frac{p^{\eta}}{{(p^{\eta}+ (1-p)^{\eta})}^{1/\eta}}$, while \citet{prelec1998probability} recommends $w(p) = \exp(-(-\ln p)^\eta)$, with $0 < \eta <1$. In both cases, the weight function has the inverted-s shape.
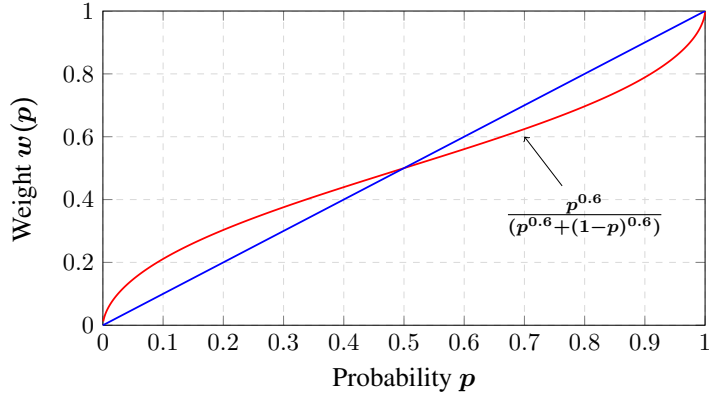
\begin{figure}[t]
\centering
\tabl{c}{
  \scalebox{0.85}{\begin{tikzpicture}
  \begin{axis}[width=11cm,height=6.5cm,legend pos=south east,
           grid = major,
           grid style={dashed, gray!30},
           xmin=0,     
           xmax=1,    
           ymin=0,     
           ymax=1,   
           axis background/.style={fill=white},
           ylabel={\large Weight $\bm{w(p)}$},
           xlabel={\large Probability $\bm{p}$}
           ]
          \addplot[domain=0:1, red, thick,smooth,samples=1500] 
             {pow(x,0.6)/(pow(x,0.6) + pow(1-x,0.6))}; 
             \node at (axis cs:  0.8,0.35) (a1) {\large $\bm{\frac{p^{0.6}}{(p^{0.6}+ (1-p)^{0.6})}}$};           
             \draw[->] (a1) -- (axis cs:  0.7,0.6);
                 \addplot[domain=0:1, blue, thick]           {x};                      
  \end{axis}
  \end{tikzpicture}}\\[1ex]
}
\caption{An example of a weight function.}
\label{fig:w}
\end{figure}

A few remarks are in order.
\begin{remark}\textit{(RL applications)}
The CPT-value, as defined in \eqref{eq:cpt-general}, has several applications in RL. In general, for any problem setting, one can define the return for a given policy and then apply CPT-functional on the return. For instance, with a fixed policy, the r.v. $X$ could be the total reward in a stochastic shortest path problem or the infinite horizon cumulative reward in a discounted MDP or the long-run average reward in an MDP - See Appendix \ref{sec:cpt-ssp} for one such application. 
\end{remark}

\begin{remark}\textit{(Generalization)}
It is easy to see that the CPT-value is a generalization of the traditional expectation, as a choice of identity map for the weight and utility functions in \eqref{eq:cpt-general} recovers the expectation of $X$.  It is also possible to get \eqref{eq:cpt-general} to coincide with risk measures (e.g. VaR and CVaR) by appropriate choice of weight functions.
\end{remark}

\begin{remark}\textit{(Sensitivity)}
Traditional EU-based approaches are sensitive to modeling errors as illustrated in the following example: 
Suppose stock $\cal{A}$ gains \$$10000$ w.p $0.001$ and loses nothing w.p. $0.999$, while stock $\cal B$ surely gains $11$. With the classic value function objective, it is optimal to invest in stock $\cal B$ as it returns $11$,  while $\cal A$ returns $10$ in expectation (assuming utility function to be the identity map). Now, if the gain probability for stock $\cal A$ was $0.002$, then it is no longer optimal to invest in stock $\cal B$ and investing in stock $A$ is optimal.
Notice that a very slight change in the underlying probabilities resulted in a big difference in the investment strategy and a similar observation carries over to a multi-stage scenario (see the house buying example in the numerical experiments section). 

Using CPT makes sense because it inflates low probabilities and thus can account for modeling errors, especially considering that model information is unavailable in practice.
Note also that in MDPs with expected utility objective, there exists a deterministic policy that is optimal. However, with CPT-value objective, the optimal policy is \textit{not necessarily} deterministic - See also the organ transplant example on pp. 75-81 of \cite{lin2013stochastic}. 
\end{remark}


\section{CPT-value estimation} 
\label{sec:cpt-sampling}


Before diving into the details of CPT-value estimation, let us discuss the conditions necessary for the CPT-value to be well-defined.
Observe that the first integral in \eqref{eq:cpt-general}, i.e., 
$\int_0^{+\infty} w^+(P(u^+(X)>z)) d z$
may diverge even if the first moment of random variable $u^+(X)$ is finite. 
For example, suppose $U$ has the tail distribution function
$P(U>z)  = \frac{1}{z^2}, z\in [1, +\infty),$
 and $w^+(z)$ takes the form $w(z) = z^{\frac{1}{3}}$. Then, the first integral in \eqref{eq:cpt-general}, i.e.,
$
\int_1^{+\infty} \frac{1}{z^{\frac{2}{3}}} dz
$
does not even exist. A similar argument applies to the second integral in \eqref{eq:cpt-general} as well.

To overcome the above integrability issues, we make different assumptions on the weight and/or utility functions. In particular, we assume that the weight functions $w^+, w^-$ are either 
\begin{inparaenum}[\it (i)]
\item Lipschitz continuous, or
\item \holder continuous, or
\item locally Lipschitz.
\end{inparaenum}
We devise a scheme for estimating \eqref{eq:cpt-general} given only samples from $X$ and show that, under each of the aforementioned assumptions, our estimator (presented next) converges almost surely. 
We also provide sample complexity bounds assuming that the utility functions are bounded.

\subsection{Estimation scheme for \holder continuous weights}
Recall the H\"{o}lder continuity property first in definition 1:
\begin{definition}\label{holder}
{\textbf{\textit{(H\"{o}lder continuity)}}}
If $0 < \alpha \leq 1$, a function $f \in C([a,b])$ is said to satisfy
a H\"{o}lder condition of order $\alpha$ (or to be H\"{o}lder continuous
of order $\alpha$) if $\exists H>0$, s.t.
\[
\sup_{x \neq y} \frac{| f(x) - f(y) |}{| x-y |^{\alpha}} \leq H .
\]
\end{definition}

In order to ensure integrability of the CPT-value \eqref{eq:cpt-general}, we make the following assumption:\\[1ex]
\textbf{Assumption (A1).}  
The weight functions $w^+, w^-$ are H\"{o}lder continuous with common order $\alpha$. Further,
$\exists \gamma \le \alpha \text{   s.t,  }$ 
$\int_0^{+\infty} P^{\gamma} (u^+(X)>z) dz < +\infty$ and $\int_0^{+\infty} P^{\gamma} (u^-(X)>z) dz < +\infty.$

The above assumption ensures that the CPT-value as defined by \eqref{eq:cpt-general} is finite - see Proposition \ref{prop:Holder-cpt-finite} in 
Appendix \ref{sec:holder-proofs} for a formal proof.

\paragraph{Approximating CPT-value using quantiles:}
Let $\xi^+_{\alpha}$ denote the $\alpha$th quantile of the r.v. $u^+(X)$. Then, it can be seen that (see Proposition \ref{prop:holder-quantile} in Appendix \ref{sec:holder-proofs})
\begin{align}
&\lim_{n \rightarrow \infty} \sum_{i=1}^{n-1} \xi^+_{\frac{i}{n}} \left(w^+\left(\frac{n-i}{n}\right)- w^+\left(\frac{n-i-1}{n}\right) \right) = \int_0^{+\infty} w^+(P(u^+(X)>z)) dz.\label{eq:holder-quant-motiv}
\end{align}
A similar claim holds with $u^-(X)$, $\xi^-_{\alpha}, w^-$ in place of  $u^+(X)$, $\xi^+_{\alpha}, w^+$, respectively. Here $\xi^-_{\alpha}$ denotes the 
$\alpha$th quantile of $u^-(X)$.

However, we do not know the distribution of $u^+(X)$ or $u^-(X)$ and hence, we next present a procedure that uses order statistics for estimating quantiles and this in turn assists estimation of the CPT-value along the lines of \eqref{eq:holder-quant-motiv}. The estimation scheme is presented in Algorithm \ref{alg:holder-est}.

\begin{algorithm}
\caption{CPT-value estimation for \holder continuous weights}
\label{alg:holder-est}
\begin{algorithmic}
\State Simulate $n$ i.i.d. samples from the distribution of $X$.
\State Order the samples and label them as follows: 
$X_{[1]}, X_{[2]}, \ldots ,X_{[n]}$. Note that $u^+(X_{[1]}),\ldots ,u^+(X_{[n]})$ are also in ascending order.
\State Denote the statistic 
\vspace{-0.5ex}
$$\overline \C_n^+:=\sum_{i=1}^{n-1} u^+(X_{[i]}) \left(w^+(\frac{n-i}{n})- w^+(\frac{n-i-1}{n}) \right).$$
\vspace{-0.5ex}
\State Apply $u^{-}$ on the sequence $\{X_{[1]}, X_{[2]}, \ldots ,X_{[n]}\}$, notice that $u^{-}(X_{[i]})$ is in descending order since $u^{-}$ is a decreasing function.     
\State Denote the statistic
\vspace{-0.5ex}
$$\overline \C_n^-:=\sum_{i=1}^{n-1} u^-(X_{[i]}) \left(w^-(\frac{i}{n})- w^-(\frac{i-1}{n}) \right). $$

\vspace{-0.5ex}
\State Return $\overline \C_n =\overline \C_n^+ - \overline \C_n^-$.
\end{algorithmic}
\end{algorithm}

\subsubsection*{Main results}
\textbf{Assumption (A2).}  The utility functions $u^+(X)$ and $u^-(X)$ are continuous and strictly increasing.

\textbf{Assumption (A2').}  In addition to (A2), the utility functions $u^+(X)$ and $u^-(X)$ are bounded above by $M<\infty$.

For the sample complexity results below, we require (A2'), while (A2) is sufficient to prove asymptotic convergence.

\begin{proposition}(\textbf{Asymptotic convergence.})
\label{prop:holder-asymptotic}
Assume (A1) and also that $F^+(\cdot)$,$F^-(\cdot)$ - the distribution functions of $u^+(X)$, and $u^-(X)$ are Lipschitz continuous with constants $L^+$ and $L^-$, respectively, on the interval $(0,+\infty)$, and 
$(-\infty, 0)$ . Then, we have that
\begin{align}
\overline \C_n
\rightarrow
\C(X)
 \text{   a.s. as } n\rightarrow \infty
\end{align}
where $\overline \C_n$ is as defined in Algorithm \ref{alg:holder-est} and $\C(X)$ as in \eqref{eq:cpt-general}.
\end{proposition}
\begin{proof}
See Appendix \ref{sec:holder-proofs}.
\end{proof}
While the above result establishes that $\overline \C_n$ is an unbiased estimator in the asymptotic sense, it is important to know the rate at which the estimate $\overline \C_n$ converges to the CPT-value $\C(X)$. 
The following sample complexity result shows that $O\left(\frac{1}{\epsilon^{2/\alpha}}\right)$ number of samples are required to be $\epsilon$-close to the CPT-value in high probability.
\begin{proposition}(\textbf{Sample complexity.})
\label{prop:holder-dkw}
Assume (A1) and (A2'). Then, $\forall \epsilon >0, \delta >0$, we have
$$
P(\left |\overline \C_n- \C(X) \right| \leq  \epsilon ) > \delta\text{     ,} \forall n \geq \ln(\frac{1}{\delta})\cdot 
\frac{4H^2 M^2}{\epsilon^{2/\alpha}}.$$
\end{proposition}
\begin{proof}
%
See Appendix \ref{sec:holder-proofs}.
\end{proof}

\subsubsection{Results for Lipschitz continuous weights}
In the previous section, it was shown that \holder continuous weights incur a sample complexity of order $O\left(\frac1{\epsilon^{2/\alpha}}\right)$ and this is higher than the canonical Monte Carlo rate of $O\left(\frac1{\epsilon^2}\right)$. In this section, we establish that one can achieve the canonical Monte Carlo rate if we consider Lipschitz continuous weights, i.e., the following assumption in place of (A1):
\todoc{How about explaining why we do this? Why do we consider this case? (Every time we do something we should explain why we do it)}
 
\textbf{Assumption (A1').}  The weight functions $w^+, w^-$ are Lipschitz with common constant $L$, and 
$u^+(X)$ and $u^-(X)$ both have bounded first moments.

Setting $\alpha=1$, one can make special cases of the claims regarding asymptotic convergence and sample complexity of Proposition \ref{prop:holder-asymptotic}--\ref{prop:holder-dkw}. However, these results are under  a restrictive Lipschitz assumption on the distribution functions of $u^+(X)$ and $u^-(X)$. Using a different proof technique that employs the dominated convergence theorem and DKW inequalities, one can obtain results similar to Proposition \ref{prop:holder-asymptotic}--\ref{prop:holder-dkw} with (A1') and (A2) only. The following claim makes this precise.

\begin{proposition}
\label{prop:lipschitz}
Assume (A1') and (A2). Then, we have that 
$$\overline \C_n
\rightarrow
\C(X)
 \text{   a.s. as } n\rightarrow \infty
$$
In addition, if we assume (A2'), we have $\forall \epsilon >0, \delta >0$ 
$$
P(\left |\overline \C_n- \C(X) \right| \leq  \epsilon ) > \delta\text{     ,} \forall n \geq \ln(\frac{1}{\delta})\cdot 
\frac{4L^2 M^2}{\epsilon^{2}}.
$$
\end{proposition}
\begin{proof}
See Appendix \ref{sec:lipschitz-proofs}.
\end{proof}

\subsection{Estimation scheme for locally Lipschitz weights and discrete $X$}
Here we assume that the r.v. $X$ is discrete valued. 
Let $p_i, i=1,\ldots,K$ denote the probability of incurring a gain/loss $x_i, i=1,\ldots,K$, where 
$x_1\le \ldots \le x_l \le 0 \le x_{l+1} \le \ldots \le x_K$ and  let
\begin{align}
\label{eq:Fk}
 F_k = 
   \sum_{i=1}^k p_k  \text{ if   } k \leq l \text{ and }
   \sum_{i=k}^K p_k  \text{ if  }  k > l.
\end{align}
Then, the CPT-value is defined as 
\begin{align*}
\C(X) = & (u^-(x_1)) w^-(p_1) 
+\sum_{i=2}^l u^-(x_i) \Big(w^-(F_i) - w^-(F_{i-1})\Big) \\
& + \sum_{i=l+1}^{K-1} u^+(x_i) \Big(w^+(F_i) - w^+(F_{i+1}) \Big)
 + u^+(x_K) w^+(p_K),
\end{align*} 
where $u^+, u^-$ are utility functions and $w^+, w^-$ are weight functions corresponding to gains and losses, respectively. The utility functions $u^+$ and $u^-$ are non-decreasing, while the weight functions are continuous, non-decreasing and have the range $[0,1]$ with $w^+(0)=w^-(0)=0$ and $w^+(1)=w^-(1)=1$. 

\paragraph{Estimation scheme.} 
Let $\hat p_k= \frac{1}{n} \sum_{i=1}^n I_{\{U =x_k\}}$ and 
\begin{align}
\label{eq:Fkhat}
 \hat F_k = 
   \sum_{i=1}^k \hat p_k  \text{ if   } k \leq l \text{ and }
   \sum_{i=k}^K \hat p_k  \text{ if  }  k > l.
\end{align}
Then, we estimate $\C(X)$ as follows:
\begin{align}
\overline \C_n = & 
u^-(x_1) w^-(\hat p_1) \!+\!\sum_{i=2}^l u^-(x_i) \Big(w^-(\hat F_i) - w^-( \hat F_{i-1})\Big) 
\nonumber\\
&
+ \sum_{i=l+1}^{K-1} u^+(x_i) \Big(w^+(\hat F_i) - w^+(\hat F_{i+1}) \Big) + u^+(x_K) w^+(\hat p_K). \label{eq:cpt-discrete-est}
\end{align}
\textbf{Assumption (A3).}  The weight functions $w^+(X)$ and $w^-(X)$ are locally Lipschitz continuous, i.e., for any $x$, there exist  $L< \infty$ and $\delta>0$, such that
$$| w^+(x) - w^+(y) | \leq L_x |x-y|, \text{ for all } y \in (x-\delta,x+\delta). $$
The main result for discrete-valued $X$ is given below.
\begin{proposition}
\label{prop:sample-complexity-discrete}
Assume (A3). Let $L=\max\{L_k, k=2...K\} $, where $L_k$ is the local Lipschitz constant of function $w^-(x)$ at points
$F_k$, where $k=1,...l$, and of function $w^+(x)$ at points $k=l+1,...K$. 
Let $A=\max\{u^{-}(x_k), k=1...l\} \bigcup \{u^{+}(x_k), k=l+1...K\}$, $\delta =\min\{\delta_k\}$, where $\delta_k$ is the half the length of the interval centered at point $F_k$ where the locally Lipschitz property with constant $L_k$ holds.
For any $\epsilon,\rho >0$, we have 
\begin{align}
P(\left|
\overline \C_n -\C(X)
\right| \leq \epsilon) > 1-\rho, \forall n> \frac{\ln(\frac{4K}{A})} { M}, 
\end{align}
where $M=\min(\delta^2, \epsilon^2/(KLA)^2)$.
\end{proposition}
In comparison to Propositions \ref{prop:holder-dkw} and \ref{prop:lipschitz}, 
observe that the sample complexity for discrete $X$ scales with the local Lipschitz constant $L$ and this can be much smaller the global Lipschitz constant of the weight functions or the weight functions may not be Lipschitz globally.  
\begin{proof}
 See Section \ref{sec:proofs-discrete}.
\end{proof}

\section{Gradient-based algorithm for CPT optimization (CPT-SPSA)}
\label{sec:1spsa}
\paragraph{Optimization objective:} 
\todoc{Is this the best place for this? How about defining policy search later, after value estimation is done?}
Suppose the r.v. $X$ in \eqref{eq:cpt-general} is a function of a $d$-dimensional parameter $\theta$. The goal then is to solve the following problem:
\begin{align}
\label{eq:opt-general}
\textrm{Find ~}\theta^* = \argmax_{\theta \in \Theta} \C(X^\theta),
\end{align}
where $\Theta$ is a compact and convex subset of $\R^d$. As mentioned earlier, the above problem encompasses policy optimization in an MDP that can be discounted or average or episodic and/or partially observed. The difference here is that we apply the CPT-functional to the return of a policy, while traditional approaches consider the expected return.

\subsection{Gradient estimation} 
Given that we operate in a learning setting and only have biased estimates of the CPT-value from Algorithm \ref{alg:holder-est}, we require a simulation scheme to estimate $\nabla \C(X^\theta)$.  
Simultaneous perturbation methods are a general class of stochastic gradient schemes that optimize a function given only noisy sample values - see \cite{Bhatnagar13SR} for a textbook introduction. SPSA is a well-known scheme that estimates the gradient using two sample values. In our context, at any iteration $n$ of CPT-SPSA-G, with parameter $\theta_n$, the gradient $\nabla \C(X^{\theta_n})$ is estimated as follows: For any  $i=1,\ldots,d$,
\begin{align}
\widehat \nabla_{i} \C(X^\theta) = \dfrac{\overline \C_n^{\theta_n+\delta_n \Delta_n} - \overline \C_n^{\theta_n-\delta_n \Delta_n}}{2 \delta_n \Delta_n^{i}},\label{eq:grad-est-spsa}
\end{align}
where $\delta_n$ is a positive scalar that satisfies (A3) below, $\Delta_n = \left( \Delta_n^{1},\ldots,\Delta_n^{d}\right)\tr$, where $\{\Delta_n^{i}, i=1,\ldots,d\}$, $n=1,2,\ldots$ are i.i.d. Rademacher, independent of $\theta_0,\ldots,\theta_n$ and $\overline \C_n^{\theta_n+\delta_n \Delta_n}$ (resp. $\overline \C_n^{\theta_n-\delta_n \Delta_n}$) denotes the CPT-value estimate that uses $m_n$ samples of the r.v. $X^{\theta_n+\delta_n \Delta_n}$ (resp. $\overline X^{\theta_n-\delta_n \Delta_n}$).
The (asymptotic) unbiasedness of the gradient estimate is proven in Lemma \ref{lemma:1spsa-bias}.


\subsection{Update rule} We incrementally update the parameter $\theta$ in the ascent direction as follows: For   $i=1,\ldots,d$,
\begin{align}
\theta^{i}_{n+1} = \Gamma_{i}\left(\theta^{i}_n + \gamma_n  \widehat \nabla_{i} \C(X^{\theta_n})\right),
\label{eq:theta-update}
\end{align}
where  $\gamma_n$ is a step-size chosen to satisfy (A3) below and
$\Gamma=\left(\Gamma_{1},\ldots,\Gamma_{d}\right)$ is an operator that ensures that the update \eqref{eq:theta-update} stays bounded within a compact and convex set $\Theta$. 
Algorithm \ref{alg:1spsa}  presents the pseudocode.

\algblock{PEval}{EndPEval}
\algnewcommand\algorithmicPEval{\textbf{\em CPT-value Estimation (Trajectory 1)}}
 \algnewcommand\algorithmicendPEval{}
\algrenewtext{PEval}[1]{\algorithmicPEval\ #1}
\algrenewtext{EndPEval}{\algorithmicendPEval}

\algblock{PEvalPrime}{EndPEvalPrime}
\algnewcommand\algorithmicPEvalPrime{\textbf{\em CPT-value Estimation (Trajectory 2)}}
 \algnewcommand\algorithmicendPEvalPrime{}
\algrenewtext{PEvalPrime}[1]{\algorithmicPEvalPrime\ #1}
\algrenewtext{EndPEvalPrime}{\algorithmicendPEvalPrime}

\algblock{PImp}{EndPImp}
\algnewcommand\algorithmicPImp{\textbf{\em Gradient Ascent}}
 \algnewcommand\algorithmicendPImp{}
\algrenewtext{PImp}[1]{\algorithmicPImp\ #1}
\algrenewtext{EndPImp}{\algorithmicendPImp}

\algtext*{EndPEval}
\algtext*{EndPEvalPrime}
\algtext*{EndPImp}
\begin{algorithm}[t]
\begin{algorithmic}
    \State {\bf Input:}  initial parameter $\theta_0 \in \Theta$ where $\Theta$ is a compact and convex subset of $\R^d$, perturbation constants $\delta_n>0$, sample sizes $\{m_n\}$, step-sizes $\{\gamma_n\}$, operator $\Gamma: \R^d \rightarrow \Theta$.
\For{$n = 0,1,2,\ldots$}	
	\State Generate $\{\Delta_n^i, i=1,\ldots,d\}$ using Rademacher distribution, independent of $\{\Delta_m, m=0,1,\ldots,n-1\}$.
	\PEval
	    \State Simulate $m_n$ samples using  $(\theta_n+\delta_n \Delta_n)$.
	    \State Obtain CPT-value estimate $\overline \C_n^{\theta_n+\delta_n \Delta_n}$. 
	    \EndPEval
	    \PEvalPrime
  	    \State Simulate $m_n$ samples using $(\theta_n-\delta_n \Delta_n)$.
	    \State Obtain CPT-value estimate $\overline \C_n^{\theta_n-\delta_n \Delta_n}$.
	    \EndPEvalPrime
	    \PImp
		\State Update $\theta_n$ using \eqref{eq:theta-update}.
		\EndPImp
\EndFor
\State {\bf Return} $\theta_n$.
\end{algorithmic}
\caption{Structure of CPT-SPSA-G algorithm.}
\label{alg:1spsa}
\end{algorithm}


\paragraph{On the number of samples $m_n$ per iteration:}
The CPT-value estimation scheme is biased, i.e., providing samples with parameter $\theta_n$ at instant $n$, we obtain its CPT-value estimate as $\C(X^{\theta_n}) + \epsilon_n^\theta$, with $\epsilon_n^\theta$ denoting the bias. The bias can be controlled by increasing the number of samples $m_n$ in each iteration of CPT-SPSA (see Algorithm \ref{alg:1spsa}). This is unlike classic simulation optimization settings where one only sees function evaluations with zero mean noise and there is no question of deciding on $m_n$ to control the bias as we have in our setting.

To motivate the choice for $m_n$, we first rewrite the update rule \eqref{eq:theta-update} as follows:
\begin{align*}
\theta^{i}_{n+1}  = & \Gamma_{i}\bigg( \theta^{i}_n +  \gamma_n \bigg( \frac{\C(X^{\theta_n +\delta_n\Delta_n}) - \C(X^{\theta_n-\delta_n\Delta_n})}{2\delta_n\Delta_n^{i}}\bigg) + \underbrace{\frac{(\epsilon_n^{\theta_n +\delta_n\Delta_n} - \epsilon_n^{\theta_n-\delta_n\Delta_n})}{2\delta_n\Delta_n^{i}}}_{\kappa_n}\bigg).
\end{align*}
Let $\zeta_n = \sum_{l = 0}^{n} \gamma_l \kappa_{l}$. Then, a critical requirement that allows us to ignore the bias term $\zeta_n$ is the following condition (see Lemma 1 in Chapter 2 of \cite{borkar2008stochastic}): 
$$\sup_{l\ge0} \left (\zeta_{n+l} - \zeta_n \right) \rightarrow 0 \text{ as } n\rightarrow\infty.$$ 
While Theorems \ref{prop:holder-asymptotic}--\ref{prop:holder-dkw} show that the bias $\epsilon^\theta$ is bounded above, to establish convergence of the policy gradient recursion \eqref{eq:theta-update}, we increase the number of samples $m_n$ so that the bias vanishes asymptotically.  The assumption below provides a condition on the increase rate of $m_n$.

\noindent\textbf{Assumption (A3).}  The step-sizes $\gamma_n$ and the perturbation constants 
$\delta_n$ are positive $\forall n$ and satisfy
\begin{align*}
\gamma_n, \delta_n \rightarrow 0, \frac{1}{m_n^{\alpha/2}\delta_n}\rightarrow 0,  \sum_n \gamma_n=\infty \text{ and } \sum_n \frac{\gamma_n^2}{\delta_n^2}<\infty. 
\end{align*}
While the conditions on $\gamma_n$ and $\delta_n$ are standard for SPSA-based algorithms, the condition on $m_n$ is motivated by the earlier discussion. 
A simple choice that satisfies the above conditions is $\gamma_n = a_0/n$, $m_n = m_0 n^\nu$ and $\delta_n = \delta_0/{n^\gamma}$, for some $\nu, \gamma >0$ with $\gamma > \nu\alpha/2$.

\subsection{Convergence result}
\begin{theorem}
\label{thm:1spsa-conv}
Assume (A1)-(A3) and also that $\C(X^\theta)$ is a continuously differentiable function of $\theta$, for any $\theta \in \Theta$\footnote{In a typical RL setting, it is sufficient to assume that the policy is continuously differentiable in $\theta$.}.
Consider the  ordinary differential equation (ODE): 
$$\dot\theta^{i}_t = \check\Gamma_{i}\left(- \nabla \C(X^{\theta^{i}_t})\right), \text{ for }i=1,\dots,d,$$ 
where 
$\check\Gamma_{i}(f(\theta)) := \lim\limits_{\alpha \downarrow 0} \frac{\Gamma_{i}(\theta + \alpha f(\theta)) - \theta}{\alpha}$, for any continuous $f(\cdot).$
 Let $\K = \{\theta \mid \check\Gamma_{i} \left(\nabla_i \C(X^{\theta})\right)=0, \forall i=1,\ldots,d\}$. Then, for $\theta_n$ governed by \eqref{eq:theta-update}, we have
$$\theta_n \rightarrow \K \text{ a.s. as } n\rightarrow \infty.$$
\end{theorem}
\begin{proof}
 See Appendix \ref{appendix:1spsa}.
\end{proof}

See Appendix \ref{sec:2spsa} for a second-order CPT-value optimization scheme based on SPSA.
\section{Simulation Experiments}
\label{sec:expts}

We consider a traffic signal control application where the aim is to improve the road user experience by an adaptive traffic light control (TLC) algorithm.
We apply the CPT-functional to the delay experienced by road users, since CPT captures realistically captures the attitude of the road users towards delays. We then optimize the CPT-value of the delay and contrast this approach with traditional expected delay optimizing algorithms.

We consider a road network with $\N$ signalled lanes that are spread across junctions and $\M$ paths, where each path connects (uniquely) two edge nodes, where the traffic is generated - see Figure \ref{fig:2x2grid}.
At any instant $n$, let $q_n^i$ and $t_n^i$ denote the queue length and elapsed time since the lane turned red, for any lane $i = 1,\ldots, \N$. Let $d_n^{i,j}$ denote the delay experienced by $j$th road user on $i$th path, for any $i=1,\ldots,\M$ and $j=1,\ldots,n_i$, where $n_i$ denotes the number of road users on path $i$.
We specify the various components of the traffic control MDP in the following.
The state $s_n=(q_n^1,\ldots,q_n^{\N},t_n^1,\ldots,t_n^{\N},d_n^{1,1},\ldots,d_n^{\M,n_{\M}})\tr$ is a vector of lane-wise queue lengths, elapsed times and path-wise delays.
The actions are the feasible sign configurations. Traffic lights that can be simultaneously switched to green form a sign configuration. 

 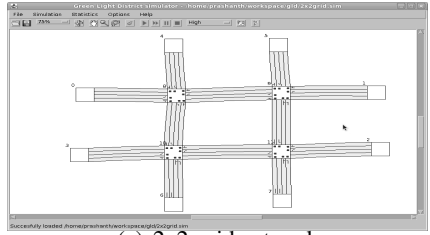
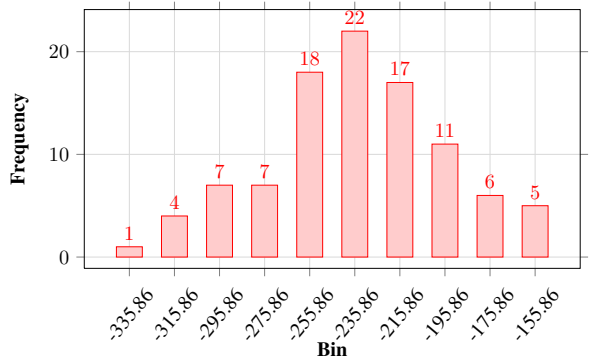
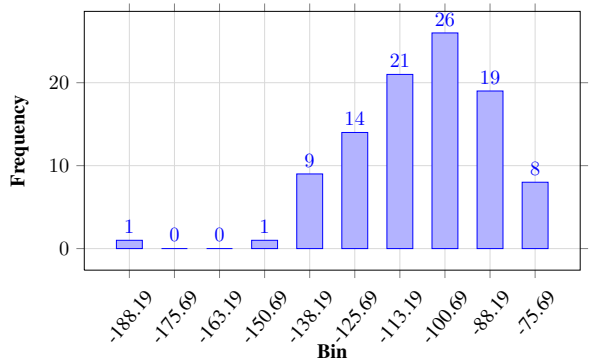
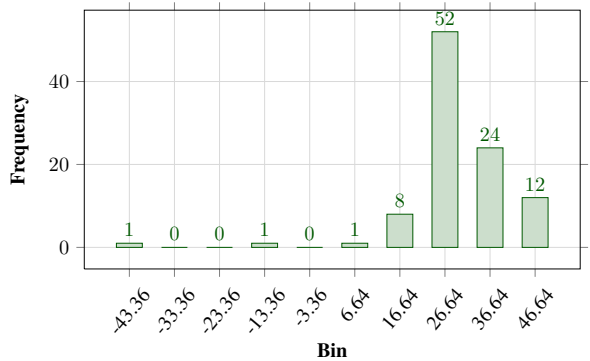
\begin{figure}
    \centering
     \begin{tabular}{c}
\subfigure[2x2-grid network]{
\label{fig:2x2grid}
        \includegraphics[width=2.2in,height=1.2in]{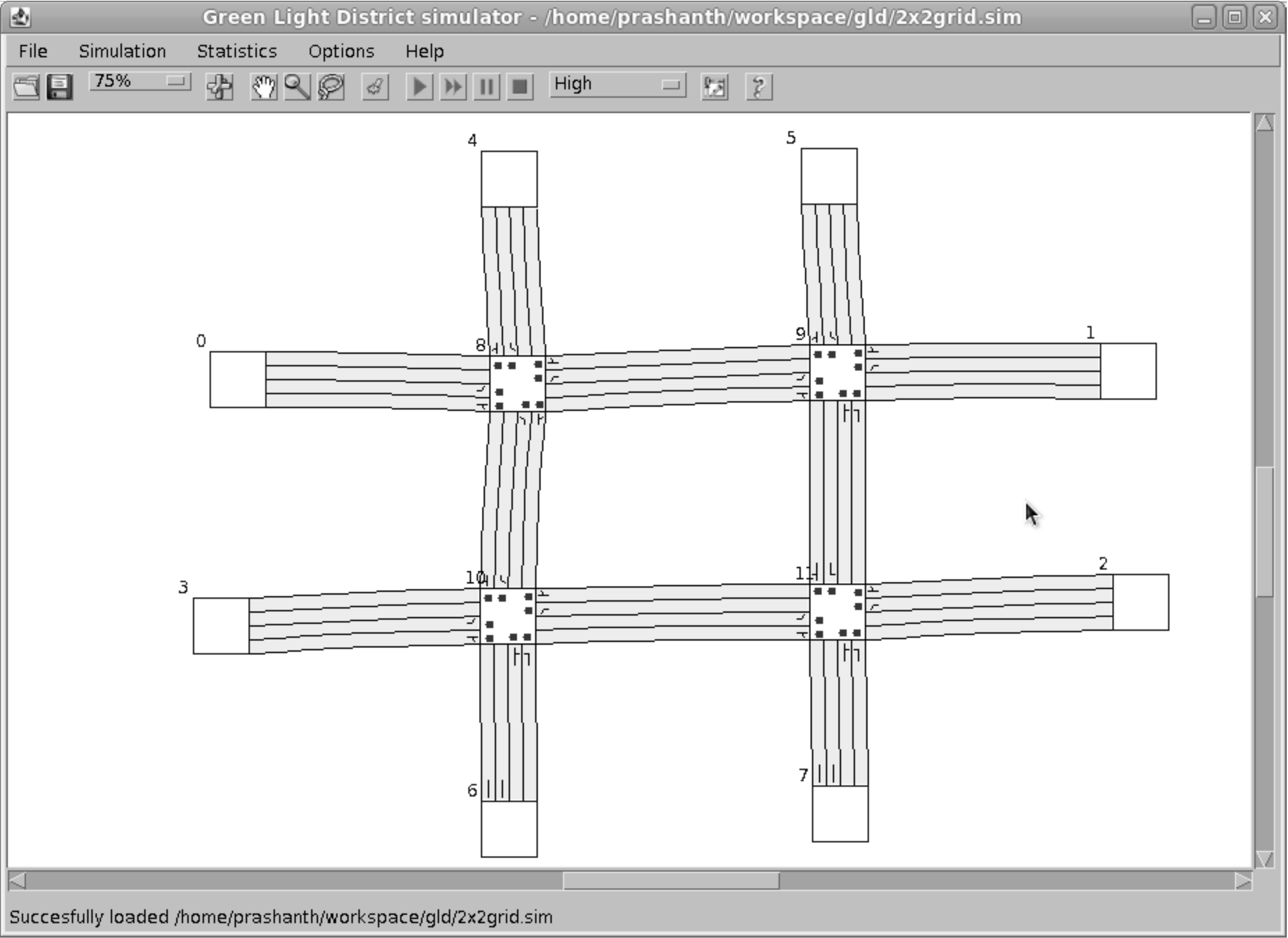}
}
\\		
\subfigure[AVG-SPSA]{
\label{fig:avg}
\tabl{c}{\scalebox{0.7}{\begin{tikzpicture}
\begin{axis}[
ybar={2pt},
legend pos=outer north east,
legend image code/.code={\path[fill=white,white] (-2mm,-2mm) rectangle
(-3mm,2mm); \path[fill=white,white] (-2mm,-2mm) rectangle (2mm,-3mm); \draw
(-2mm,-2mm) rectangle (2mm,2mm);},
ylabel={\bf Frequency},
xlabel={\textbf{Bin}},
x label style={at={(axis description cs:0.5,-0.15)},anchor=north},
symbolic x coords={0, -335.86, -315.86, -295.86, -275.86, -255.86, -235.86, -215.86, -195.86, -175.86, -155.86, 11},
xmin={0},
xmax={11},
xtick=data,
ytick align=outside,
xticklabel style={rotate=50, align=center},
bar width=14pt,
nodes near coords,
grid,
grid style={gray!30},
width=11cm,
height=6.5cm,
]
\addplot[red, fill=red!20]   coordinates {  (-335.86,1) (-315.86,4) (-295.86,7) (-275.86,7) (-255.86,18) (-235.86,22) (-215.86,17) (-195.86,11) (-175.86,6) (-155.86,5)}; 
\end{axis}
\end{tikzpicture}}\\[1ex]}
}
\\
\subfigure[EUT-SPSA]{
\label{fig:eut}
\tabl{c}{\scalebox{0.7}{\begin{tikzpicture}
\begin{axis}[
ybar={2pt},
legend pos=outer north east,
legend image code/.code={\path[fill=white,white] (-2mm,-2mm) rectangle
(-3mm,2mm); \path[fill=white,white] (-2mm,-2mm) rectangle (2mm,-3mm); \draw
(-2mm,-2mm) rectangle (2mm,2mm);},
ylabel={\bf Frequency},
xlabel={\textbf{Bin}},
x label style={at={(axis description cs:0.5,-0.15)},anchor=north},
symbolic x coords={0,-188.19,-175.69,-163.19,-150.69,-138.19,-125.69,-113.19,-100.69,-88.19,-75.69,11},
xmin={0},
xmax={11},
xtick=data,
ytick align=outside,
xticklabel style={rotate=50, align=center},
bar width=14pt,
nodes near coords,
grid,
grid style={gray!30},
width=11cm,
height=6.5cm,
]
\addplot   coordinates {  (-188.19,1) (-175.69,0) (-163.19,0) (-150.69,1) (-138.19,9) (-125.69,14) (-113.19,21) (-100.69,26) (-88.19,19) (-75.69,8) }; 
\end{axis}
\end{tikzpicture}}\\[1ex]}
}
\\
\subfigure[CPT-SPSA]{
\label{fig:cpt}
\tabl{c}{\scalebox{0.7}{\begin{tikzpicture}
\begin{axis}[
ybar={2pt},
legend pos=outer north east,
legend image code/.code={\path[fill=white,white] (-2mm,-2mm) rectangle
(-3mm,2mm); \path[fill=white,white] (-2mm,-2mm) rectangle (2mm,-3mm); \draw
(-2mm,-2mm) rectangle (2mm,2mm);},
ylabel={\bf Frequency},
xlabel={\textbf{Bin}},
x label style={at={(axis description cs:0.5,-0.15)},anchor=north},
symbolic x coords={0, -43.36,-33.36,-23.36,-13.36,-3.36,6.64,16.64,26.64,36.64,46.64, 11},
xmin={0},
xmax={11},
xtick=data,
ytick align=outside,
xticklabel style={rotate=50, align=center},
bar width=14pt,
nodes near coords,
grid,
grid style={gray!30},
width=11cm,
height=6.5cm,
]
\addplot[darkgreen, fill=darkgreen!20]   coordinates {  (-43.36,1) (-33.36,0) (-23.36,0) (-13.36,1) (-3.36,0) (6.64,1) (16.64,8) (26.64,52) (36.64,24) (46.64,12) }; 
\end{axis}
\end{tikzpicture}}\\[1ex]}
}
\end{tabular}
\caption{Histogram of CPT-value of the average delay for three different algorithms (all based on SPSA): AVG uses plain sample means (no utility/weights), EUT uses utilities but no weights and CPT uses both utilites and weights. Note: larger values are better.}
\label{fig:histogram-perf}
\end{figure}

We consider three different notions of return as follows:

\textbf{CPT:} Let $\mu^i$ be the proportion of road users along path $i$, for $i=1,\ldots,\M$. Any road user along path $i$, will evaluate the delay he experiences in a manner that is captured well by CPT. Let $X_i$ be the delay r.v. for path $i$ and let the corresponding CPT-value be $\C(X_i)$. With the objective of maximizing the experience of road users across paths, the overall return to be optimized is given by
\begin{align}
\text{CPT}(X_1,\ldots,X_{\M}) = \sum_{i=1}^{\M} \mu^i \C(X_i).\label{eq:cpt-traffic}
\end{align}
\textbf{EUT:} Here we only use the utility functions $u^+$ and $u^-$ to handle gains and losses, but do not distort probabilities. 
Thus, the EUT objective is defined as\\
\centerline{$\text{EUT}(X_1,\ldots,X_{\M}) = \sum_{i=1}^{\M} \mu^i \left(\E(u^+(X_i) - \E(u^-(X_i)\right),$}
	
where $\E(u^+(X_i)) = \intinfinity P(u^+(X_i)>z) dz$ and $\E(u^-(X_i)) - \intinfinity P(u^-(X_i)>z) dz$, for $i=1,\ldots,\M$.

\textbf{AVG:} This is similar to EUT, except that no distinction between gains and losses via utility functions nor distort using weights as in CPT. Thus, $\text{AVG}(X_1,\ldots,X_{\M}) = \sum_{i=1}^{\M} \mu^i \E(X_i)$. 

An important recommendation of CPT is to employ a reference point to calculate gains and losses. 
In our setting, we use path-wise delays obtained from a pre-timed TLC (cf. the Fixed TLCs in \cite{prashanth2011reinforcement}) as the reference point. In other words, if the delay of any algorithm (say CPT-SPSA) is less than that of pre-timed TLC, then the (positive) difference in delays is perceived as a gain and in the complementary case, the delay difference is perceived as a loss. The $d_n^{i,j}$ in the state $s_n$ are to be understood as the delay difference to the pre-timed TLC.  

The underlying policy in all the algorithms that we implement follows a Boltzmann distribution  and has the form
$
\pi_{\theta}(s,a) = \frac{e^{\theta^{\top} \phi_{s,a}}}{\sum_{a' \in {\A(s)}} e^{\theta^{\top} \phi_{s,a'}}},
\hspace{6pt} \forall s \in \S,\;\forall a \in \A(s),
$
where the features $\phi(s,a)$ are chosen as in \cite{prashanth2012threshold}.

We implement the following TLC algorithms:

{\bf\em CPT-SPSA}: This is the first-order algorithm with SPSA-based gradient estimates, as described in Algorithm \ref{alg:1spsa}. In particular, the estimation scheme in Algorithm \ref{alg:holder-est} is invoked to estimate $\C(X_i)$ for each path $i=1,\ldots,\M$, with $d_n^{i,j}, j=1,\ldots,n_i$ as the samples.

{\bf\em EUT-SPSA}: This is similar to CPT-SPSA, except that weight functions $w^+(p)=w^-(p)=p,$ for $p\in [0,1]$. 

{\bf\em AVG-SPSA}: This is similar to CPT-SPSA, except that weight functions $w^+(p)=w^-(p)=p,$ for $p\in [0,1]$. 

For both CPT-SPSA and EUT-SPSA, we set the utility functions (see \eqref{eq:cpt-general}) as follows:
$u^+(x) =  |x|^{\sigma}$, and  $u^-(x) = \lambda |x|^{\sigma}$, with $\lambda = 2.25$ and $\sigma = 0.88$.
For CPT-SPSA, we set the weights as follows:
$w^+(p) = \frac{p^{\eta_1}}{{(p^{\eta_1}+ (1-p)^{\eta_1})}^{\frac{1}{\eta_1}}}$ and  
$w^-(p) = \frac{p^{\eta_2}}{{(p^{\eta_2}+ (1-p)^{\eta_2})}^{\frac{1}{\eta_2}}}$, with
$\eta_1 = 0.61$ and $\eta_2 = 0.69$. These choices are based on median estimates given by \cite{tversky1992advances} and have been used earlier in a traffic application (see \cite{gao2010adaptive}).
For all the algorithms, we set $\delta_n = 1.9/n^{0.101}$ and $a_n = 1/(n+50)$ and this is motivated by standard guidelines - see \cite{spall2005introduction}. The initial point $\theta_0$ is the $d$-dimensional vector of ones and $\forall i$, the operator $\Gamma_i$ projects $\theta_i$ onto the set $[0.1, 10.0]$.
    
The experiments involve two phases.
A training phase where we run each algorithm for $200$ iterations, with each iteration involving two perturbed simulations, each of trajectory length $500$. This is followed by a test phase where we fix the policy for each algorithm and $100$ independent simulations of the MDP (each with a trajectory length of $1000$) are performed. After each run in the test phase, the overall CPT-value \eqref{eq:cpt-traffic} is estimated. 

Figures \ref{fig:avg}--\ref{fig:cpt} present the histogram of the CPT-values from the test phase for AVG-SPSA, EUT-SPSA and CPT-SPSA, respectively.  A similar exercise for pre-timed TLC resulted in a CPT-value of $-46.14$. It is evident that each algorithm converges to a different policy. However, the CPT-value of the resulting policies is highest in the case of CPT-SPSA followed by EUT-SPSA and AVG-SPSA in that order. Intuitively this is expected because AVG-SPSA uses neither utilities nor probability distortions, while EUT-SPSA distinguishes between gains and losses using utilities while not using weights to distort probabilities.
The results in Figure \ref{fig:histogram-perf} argue for specialized algorithms that incorporate CPT-based criteria, esp. in the light of previous findings which show CPT matches human evaluation well and there is a need for algorithms that serve human needs well.


\section{Conclusions and Future Work}
\label{sec:conclusions}
CPT has been a very popular paradigm for modeling human decisions among psychologists/economists, but has escaped the radar of the AI community. This work is the first step in incorporating CPT-based criteria into an RL framework. However, both estimation and control of CPT-based value is challenging. 
We proposed a quantile-based estimation scheme that converges at the optimal rate. Next, for the problem of control, since CPT-value does not conform to any Bellman equation, we employed SPSA - a popular simulation optimization scheme and designed a first-order algorithm for optimizing the CPT-value. 
We provided theoretical convergence guarantees for all the proposed algorithms and illustrated the usefulness of CPT-based criteria in a traffic signal control application.


\section*{Appendix}


\appendix


\section{Background on CPT}
\label{sec:appendix-cpt-intro}
\todoc[inline]{Is this section referred to from the main text?}
For a random variable $X$, let $p_i, i=1,\ldots,K$ denote the probability of incurring a gain/loss $x_i, i=1,\ldots,K$. 
Given a utility function $u$ and weighting function $w$, \textit{\textbf{Prospect theory}} (PT) value is defined as $\C(X) = \sum_{i=1}^K u(x_i) w(p_i)$. 
The idea is to take an utility function that is $S$-shaped, so that it satisfies the \textit{diminishing sensitivity}  property. 
If we take the weighting function $w$ to be the identity, then one recovers the classic expected utility. A general weight function inflates low probabilities and deflates high probabilities and this has been shown to be close to the way humans make decisions (see \cite{kahneman1979prospect}, \cite{fennema1997original} for a justification, in particular via empirical tests using human subjects).
However, PT is lacking in some theoretical aspects as it violates first-order \textit{stochastic dominance}. Consider the following example from \cite{fennema1997original}: Suppose there are $20$ prospects (outcomes) ranging from $-10$ to $180$, each with probability $0.05$. If the weight function is such that $w(0.05) > 0.05$, then it uniformly overweights all \textit{low-probability} prospects and the resulting PT value is higher than the expected value $85$. This violates stochastic dominance, since a shift in the probability mass from bad outcomes did not result in a better prospect. 

\textbf{Cumulative prospect theory} (CPT) \cite{tversky1992advances} uses a similar measure as PT, except that the weights are a function of cumulative probabilities. First, separate the gains and losses as 
$x_1\le \ldots \le x_l \le 0 \le x_{l+1} \le \ldots \le x_K$. Then, the CPT-value is defined as 
\begin{align*}
\C(X) = &(u^-(x_1))\cdot w^-(p_1) 
+\sum_{i=2}^l u^-(x_i) \Big(w^-(\sum_{j=1}^i p_j) - w^-(\sum_{j=1}^{i-1} p_j)\Big) 
\\&
 + \sum_{i=l+1}^{K-1} u^+(x_i) \Big(w^+(\sum_{j=i}^K p_j) - w^+(\sum_{j=i+1}^K p_j) \Big)
 + u^+(x_K)\cdot w^+(p_K), 
\end{align*} 
where $u^+, u^-$ are utility functions and $w^+, w^-$ are weight functions corresponding to gains and losses, respectively. The utility functions $u^+$ and $u^-$ are non-decreasing, while the weight functions are continuous, non-decreasing and have the range $[0,1]$ with $w^+(0)=w^-(0)=0$ and $w^+(1)=w^-(1)=1$ . 
Unlike PT, the CPT-value does not violate stochastic dominance. In the aforementioned example, increasing $w^-(0.05)$ and $w^+(0.05)$ does not impact outcomes other than those on the extreme, i.e., $-10$ and $180$, respectively. For instance, the weight for outcome $100$ would be $w^+(0.45) - w^+(0.40)$. Thus, CPT formalizes the intuitive notion that humans are sensitive to extreme outcomes and relatively insensitive to intermediate ones.

\subsection*{Allais paradox}
Suppose we have the following two traffic light switching policies:

\textbf{\textit{[Policy 1]}} A throughput (number of vehicles that reach destination per unit time) of $1000$  w.p. $1$. Let this be denoted by $(1000,1)$.

\textbf{\textit{[Policy 2]}}  $(10000, 0.1; 1000,0.89; 100, 0.01)$ i.e., throughputs $10000$, $1000$ and $100$ with respective probabilities $0.1$, $0.89$ and $0.01$.

Humans usually choose Policy $1$ over Policy $2$. On the other hand, consider the following two policies:

\textbf{\textit{[Policy 3]}} (100,0.89; 1000, 0.11)

\textbf{\textit{[Policy 4]}} (100,0.9; 10000, 0.1)

Humans usually choose Policy $4$ over Policy $3$. 

We can now argue against using expected utility (EU) as an objective as follows: Let $u$ be the utility function in EU.
\begin{align}
&\text{Policy 1 is preferred over Policy 2}\nonumber\\ 
&\Rightarrow u(1000) > 0.1 u(10000) + 0.89 u(1000) + 0.01 u(100)\nonumber\\
&\Rightarrow 0.11 u(1000) > 0.1 u(10000) + 0.01 u(100) \label{eq:12}\\[1ex]
&\text{Policy 4 is preferred over Policy 3}\nonumber\\ 
&\Rightarrow 0.89 u(100) + 0.11 u(1000) < 0.9 u(100) + 0.1 u(10000)\nonumber\\
&\Rightarrow 0.11 u(1000) < 0.1 u(10000) + 0.01 u(100) \label{eq:23}
\end{align}

And we have a contradiction from \eqref{eq:12} and \eqref{eq:23}.

\section{CPT-value in a Stochastic Shortest Path Setting}
\label{sec:cpt-ssp}
We consider a stochastic shortest path (SSP) problem with states $\S=\{0,\ldots,\L\}$, where $0$ is a special reward-free absorbing state.  A randomized policy $\pi$ is a function that maps any state $s\in \S$ onto a probability distribution over the actions $\A(s)$ in state $s$. As is standard in policy gradient algorithms, we parameterize $\pi$ and assume it is continuously differentiable in its parameter $\theta \in \R^d$.  
An \textit{episode} is a simulated sample path using policy $\theta$ that starts in state $s^0\in \S$, visits $\{s_1,\ldots, s_{\tau-1}\}$ before ending in the absorbing state $0$, where $\tau$ is the first passage time to state $0$.
Let $D^\theta(s^0)$ be a random variable (r.v) that denote the total reward from an episode, defined by
$$ D^\theta(s^0) = \sum\limits_{m=0}^{\tau-1} r(s_m,a_m), $$
where the actions $a_m$ are chosen using policy $\theta$ and $r(s_m, a_m)$ is the single-stage reward in state $s_m\in \S$ when action $a_m \in \A(s_m)$ is chosen. 

Instead of the traditional RL objective for an SSP of maximizing the expected value $\E (D^\theta(s^0))$, 
we adopt the CPT approach and aim to solve the following problem: 
$$ \max_{\theta \in \Theta} \C(D^\theta(s^0)),$$
where $\Theta$ is the set of admissible policies that are \textit{proper}\footnote{A policy $\theta$ is proper if $0$ is recurrent and all other states are transient for the Markov chain underlying $\theta$. It is standard to assume that policies are proper in an SSP setting - cf. \cite{bertsekas1995dynamic}.} and the CPT-value function $\C(D^\theta(s^0))$ is defined as
\begin{align}
\C(D^\theta(s^0))& = \intinfinity w^+(P(u^+(D^\theta(s^0)))>z) dz \nonumber
\\&- \intinfinity w^-(P(u^-(D^\theta(s^0)))>z) dz. \label{eq:cpt-mdp}
\end{align}

\section{Proofs for CPT-value estimator}
\label{appendix:cpt-est}

\subsection{\holder continuous weights}
\label{sec:holder-proofs}
For proving Proposition \ref{prop:holder-asymptotic} and \ref{prop:sample-complexity-discrete}, we require Hoeffding's inequality, which is given below.
\begin{lemma}
Let $Y_1,...Y_n$ be independent random variables satisfying $P(a\leq Y_i \leq b)= 1,$ for each $i$, where $a<b.
$ Then for $t>0$,
$$P(\left|\sum_{i=1}^n Y_i -\sum_{i=1}^n E(Y_i)\right| \geq nt ) \leq 2\exp{\{-2nt^2 /(b-a)^2\}}. $$
\end{lemma}

\begin{proposition}
\label{prop:Holder-cpt-finite}
Under (A1'), the CPT-value $\C(X)$ as defined by \eqref{eq:cpt-mdp} is finite. 
\end{proposition}
\begin{proof}

H\"{o}lder continuity of $w^+$ together with the fact that $w^+(0)=0$ imply that 
$$
\int_0^{+\infty} w^+(P(u^{+}(X)>t)) dz 
\le H \int_0^{+\infty} P^{\alpha} (u^+(X)>z) dz
\le H \int_0^{+\infty} P^{\gamma} (u^+(X)>z) dz 
<+\infty.
$$
The second inequality is valid since $P(u^+(X)>z) \leq 1$. The claim follows for the first integral in \eqref{eq:cpt-mdp} and the finiteness of the second integral in \eqref{eq:cpt-mdp} can be argued in an analogous fashion.
\end{proof}

\begin{proposition}
\label{prop:holder-quantile}
Assume (A1'). Let $\xi^+_{\frac{i}{n}}$ and $\xi^-_{\frac{i}{n}}$ denote the $\frac{i}{n}$th quantile of $u^+(X)$ and $u^-(X)$, respectively. Then, we have 
\begin{align}
\label{eq:simple-estimation}
\begin{split}
\lim_{n \rightarrow \infty} \sum_0^{n-1} \xi^+_{\frac{i}{n}} (w^+(\frac{n-i}{n})- w^+(\frac{n-i-1}{n}) ) = \int_0^{+\infty} w^+(P(u^+(X)>z)) dz < +\infty,
\\
\lim_{n \rightarrow \infty} \sum_0^{n-1} \xi^-_{\frac{i}{n}} (w^-(\frac{n-i}{n})- w^-(\frac{n-i-1}{n}) ) = \int_0^{+\infty} w^-(P(u^-(X)>z)) dz < +\infty
\end{split}
\end{align}
\end{proposition}

\begin{proof}
We shall focus on proving the first part of equation \eqref{eq:simple-estimation}. Consider the following linear combination of simple functions: 
\begin{align}
\sum_{i=0}^{n-1} w^+ (\frac{i}{n}) 
\cdot I_{[\xi^+_\frac{n-i-1}{n}, \xi_\frac{n-i}{n}]}(t),
\label{eq:simplew}
\end{align}
which will converge almost everywhere to the function $w(P(u^{+}(X)>t))$ in the interval $[0, +\infty)$, and also notice that 
\begin{align}
\sum_{i=0}^{n-1} w^+ (\frac{i}{n}) 
\cdot I_{[\xi^+_\frac{n-i-1}{n}, \xi^+_\frac{n-i}{n}]}(t)
<
w(P(u^{+}(X)>t)),
\text{         } \forall t \in [0,+\infty).
\end{align}

The integral of \eqref{eq:simplew} can be simplified as follows:
\begin{align}
 \int_0^{+\infty} \sum_{i=0}^{n-1} w^+_{\frac{i}{n}} \cdot I_{[\xi^+_\frac{n-i-1}{n},
\xi^+_\frac{n-i}{n}]}(t)  & = \sum_{i=0}^{n-1} w^+_{\frac{i}{n}}(t) \cdot (\xi^+(\frac{n-i}{n}) -
\xi^+(\frac{n-i-1}{n})) \\ & = \sum_{i=0}^{n-1} \xi^+_{\frac{i}{n}} \cdot (w^+_(\frac{n-i}{n})-
    w^+(\frac{n-i-1}{n})).
\end{align}
The H\"{o}lder continuity property assures the fact that 
$\lim_{n \rightarrow \infty}  | w^+(\frac{n-i}{n})- w^+(\frac{n-i-1}{n})| =0$, and the limit in \eqref{eq:simple-estimation} holds through a typical application of the dominated convergence theorem.
The second part of \eqref{eq:simple-estimation} can be justified in a similar fashion.
\end{proof} 

\subsection*{Proof of Proposition \ref{prop:holder-asymptotic}}

\begin{proof}
Without loss of generality, assume that \holder constant $H$ is  $1$.
We first prove  that 
$$\overline \C^{+}_n
\rightarrow
\C^{+}(X)
 \text{   a.s. as } n\rightarrow \infty.$$
Or equivalently, show that
\begin{align}
\lim_{n\rightarrow +\infty} \sum_{i=1}^{n-1} u^+(X_{[i]}) (w^+(\frac{n-i+1}{n})- w^+(\frac{n-i}{n}))
&\xrightarrow{n \rightarrow\infty} \int_0^{+\infty} w^+(P(U>t)) dt , \text{w.p. } 1
\label{eq:claim11}
\end{align}

The main part of the proof is concentrated on finding an upper bound of the probability
\begin{align}
P ( \left| \sum_{i=1}^{n-1} u^+(X_{[i]}) \cdot (w^+(\frac{n-i}{n} )  - w^+(\frac{n-i-1}{n} ) ) -
\sum_{i=1}^{n-1} \xi^+_{\frac{i}{n}} \cdot (w^+(\frac{n-i}{n} )  - w^+(\frac{n-i-1}{n} ) ) \right| >
\epsilon),
\end{align}
for any given $\epsilon>0$.
Observe that
\begin{align}
& P ( \left| \sum_{i=1}^{n-1} u^+(X_{[i]}) \cdot (w^+(\frac{n-i}{n} )  - w^+(\frac{n-i-1}{n} ) ) -
\sum_{i=1}^{n-1} \xi^+_{\frac{i}{n}} \cdot (w^+(\frac{n-i}{n} )  - w^+(\frac{n-i-1}{n} ) ) \right| >
\epsilon) \nonumber\\ 
& \leq P ( \bigcup _{i=1}^{n-1} \{ \left| u^+(X_{[i]}) \cdot (w^+_(\frac{n-i}{n}) -
w^+{(\frac{n-i-1}{n})}) - \xi^+_{\frac{i}{n}} \cdot (w^+(\frac{n-i}{n} )  - w^+(\frac{n-i-1}{n} ) )
\right| > \frac{\epsilon}{n} \}) \nonumber\\ 
& \leq \sum _{i=1}^{n-1} P ( \left| u^+(X_{[i]}) \cdot
(w^+(\frac{n-i}{n} )  - w^+(\frac{n-i-1}{n} ) ) - \xi^+_{\frac{i}{n}} \cdot (w^+_(\frac{n-i}{n}) -
w^+_(\frac{n-i-1}{n})) \right| > \frac{\epsilon}{n}) \\ & = \sum _{i=1}^{n-1} P ( \left| ( u^+(X_{[i]}) -
\xi^+_{\frac{i}{n}}) \cdot (w^+(\frac{n-i}{n} )  - w^+(\frac{n-i-1}{n} ) ) \right| > \frac{\epsilon}{n})
\nonumber\\ 
& \leq \sum _{i=1}^{n-1} P ( \left| ( u^+(X_{[i]}) - \xi^+_{\frac{i}{n}}) \cdot \cdot (\frac{1}{n})^{\alpha}
\right| > \frac{\epsilon}{n}) \nonumber\\ 
& = \sum _{i=1}^{n-1} P ( \left| ( u^+(X_{[i]}) - \xi^+_{\frac{i}{n}})
\right| > \frac{\epsilon}{\cdot n^{1-\alpha}}).\label{eq:bd12}
\end{align}
Now we find the upper bound of the probability of a single item in the sum above, i.e.,
\begin{align*}
& P( \left | u^+(X_{[i]}) - \xi^+_{\frac{i}{n}} \right | > \frac {\epsilon} {n^{(1-\alpha)}}) \\ & = P (
    u^+(X_{[i]}) - \xi^+_{\frac{i}{n}} > \frac {\epsilon} {n^{(1-\alpha)}}) + P ( u^+(X_{[i]}) -
    \xi^+_{\frac{i}{n}} < - \frac {\epsilon} {n^{(1-\alpha)}}).
\end{align*} 

\noindent We focus on the term 
$
P ( u^+(X_{[i]}) - \xi^+_{\frac{i}{n}} > \frac {\epsilon}{\nalpha})
$.
Let $W_t = I_{(u^+(X_t) > \xi^+_{\frac{i}{n}} + \frac{\epsilon}{n^{(1-\alpha)}})}, t=1, \ldots,n.$ Using the fact that probability distribution function is non-decreasing, we obtain 
\begin{align*}
 P ( u^+(X_{[i]}) - \xi^+_{\frac{i}{n}} > \frac {\epsilon}{\nalpha})  & = P ( \sum _{t=1}^{n} W_t >
n\cdot(1-\frac{i}{n^{(1-\alpha)}})) \\ & = P ( \sum _{t=1}^{n} W_t - n \cdot [1-F^{+}(\xi^+_{\frac{i}{n}}
+\frac{\epsilon}{n^{(1-\alpha)}})] > n \cdot [F^{+}(\xi^+_{\frac{i}{n}} +\frac{\epsilon}{n^{(1-\alpha)}})
- \frac{i}{n}]).
\end{align*}

Using the fact that 
$E W_t = 1-F^{+}(\xi^+_{\frac{i}{n}} +\frac{\epsilon}{n^{(1-\alpha)}})$ in conjunction with Hoeffding's inequality, we obtain
\begin{align}
P ( \sum _{i=1}^{n} W_t - n \cdot [1-F^{+}(\xi^+_{\frac{i}{n}} +\frac{\epsilon}{n^{(1-\alpha) }  } ) ] > n
\cdot [F^{+}(\xi^+_{\frac{i}{n}} +\frac{\epsilon}{n^{(1-\alpha)} } ) - \frac{i}{n}]) < e^{-2n\cdot
\delta^{'}_t},
\end{align}
where $\delta^{'}_i = F^{+}(\xi^+_{\frac{i}{n}} +\frac{\epsilon} {n^{(1-\alpha)} }) - \frac{i}{n}$. Since 
$F^{+}(x)$ is Lipschitz, we have that $ \delta^{'}_i \leq L^{+} \cdot (\frac{\epsilon}{\nalpha})$.
Hence, we obtain
\begin{align}
P ( u^+(X_{[i]}) - \xi^+_{\frac{i}{n}} > \frac {\epsilon}{\nalpha}) < e^{-2n\cdot L^{+}
\frac{\epsilon}{\nalpha} } = e^{-2n^\alpha \cdot L ^{+} \epsilon}
\label{eq:a12}
\end{align}
In a  similar fashion, one can show that 
\begin{align}
P ( u^+(X_{[i]}) -\xi^+_{\frac{i}{n}} < -\frac {\epsilon} {\nalpha}) \leq e^{-2n^\alpha \cdot L^{+}  \epsilon}
\label{eq:a23}
\end{align}
Combining \eqref{eq:a12} and \eqref{eq:a23}, we obtain
\begin{align*}
P ( \left| u^+(X_{[i]}) -\xi^+_{\frac{i}{n}} \right| < -\frac {\epsilon} {\nalpha}) \leq 2\cdot
e^{-2n^\alpha \cdot L^{+} \epsilon} , \text{   }\forall i\in \mathbb{N} \cap (0,1) 
\end{align*}
Plugging the above in \eqref{eq:bd12}, we obtain
\begin{align}
&
P ( \left| \sum_{i=1}^{n-1} u^+(X_{[i]}) \cdot (w^+(\frac{n-i}{n} )  - w^+(\frac{n-i-1}{n} ) ) -
\sum_{i=1}^{n-1} \xi^+_{\frac{i}{n}} \cdot (w^+(\frac{n-i}{n} )  - w^+(\frac{n-i-1}{n} ) ) \right| >
\epsilon) \nonumber\\
&\leq 2n\cdot e^{-2n^\alpha \cdot L^{+}}.\label{eq:holder-sample-complexity-extract}
\end{align}

Notice that $\sum_{n=1}^{+\infty}  2n \cdot e^{-2n^{\alpha}\cdot L^{+} \epsilon}< \infty$ since the sequence 
$2n \cdot e^{-2n^{\alpha}\cdot L^{+}}$ will decrease more rapidly than the sequence
$\frac{1}{n^k}$, $\forall k>1$.

By applying the Borel Cantelli lemma, we have that $\forall \epsilon >0$
$$
P ( \left| \sum_{i=1}^{n-1} u^+(X_{[i]}) \cdot (w^+(\frac{n-i}{n} )  - w^+(\frac{n-i-1}{n} ) ) -
\sum_{i=1}^{n-1} \xi^+_{\frac{i}{n}} \cdot (w^+(\frac{n-i}{n} )  - w^+(\frac{n-i-1}{n} ) ) \right| >
\epsilon , i.o.) =0, $$
which implies 
$$
\sum_{i=1}^{n-1} u^+(X_{[i]}) \cdot (w^+(\frac{n-i}{n} )  - w^+(\frac{n-i-1}{n} ) ) - \sum_{i=1}^{n-1}
\xi^+_{\frac{i}{n}} \cdot (w^+(\frac{n-i}{n} )  - w^+(\frac{n-i-1}{n} ) ) \xrightarrow{n \rightarrow
+\infty} 0 \text{   w.p } 1 ,
$$
which proves \eqref{eq:claim11}. 

The proof of 
$\C_{n}^{-} \rightarrow \C^{-}(X)$ follows in a similar manner as above by replacing $u^+(X_{[i]})$ by $u^-(X_{[n-i]})$, after observing that $u^{-}$ is decreasing, which in turn implies that
$u^-(X_{[n-i]})$ is an estimate of the quantile $\xi^{-}_{\frac{i}{n}}$.
\end{proof}

\subsection*{Proof of Proposition \ref{prop:holder-dkw}}
For proving Proposition \ref{prop:holder-dkw}, we require the following well-known inequality that provide a finite-time bound on the distance between empirical distribution and the true distribution:
\begin{lemma}{\textbf{\textit{(Dvoretzky-Kiefer-Wolfowitz (DKW) inequality)}}}\\
Let ${\hat F_n}(u)=\frac{1}{n} \sum_{i=1}^n 1_{((u(X_i)) \leq u)}$ denote the empirical distribution of a r.v. $U$, with $u(X_1),\ldots,u(X_n)$ being sampled from the r.v $u(X)$.
The, for any $n$ and $\epsilon>0$, we have
$$
P(\sup_{x\in \mathbb{R}}|\hat{F_n}(x)-F(x)|>\epsilon ) \leq 2 e^{-2n\epsilon^2}.
$$
\end{lemma}

The reader is referred to Chapter 2 of \cite{wasserman2006} for more on empirical distributions in general and DKW inequality in particular.

\begin{proof}
We prove the $w^+$ part, and the $w^-$ part follows in a similar fashion.
Since $u^+(X)$ is bounded above by $M$ and $w^+$ is H\"{o}lder-continuous, we have
\begin{align*}
&\left|\int_0^{\infty} w^+(P(u^+(X))>t) dt- \int_0^{\infty} w^+(1- {\hat F^+_n}(t)) dt\right| \\ = &
    \left|\int_0^M w^+(P(u^+(X))>t) dt- \int_0^M w^+(1- {\hat F^+_n}(t)) dt\right| \\
\leq& \left|\int_0^M H\cdot |P(u^+(X)<t)-{\hat F^+_n}(t)|^\alpha dt\right|\\ \leq& HM\sup_{x\in
\mathbb{R}}\left|P(u^+(X)<t)-{\hat F^+_n}(t)\right|^\alpha.
\end{align*}
Now, plugging in the DKW inequality, we obtain
\begin{align}
&
P\left(\left|\intinfinity w^+(P(u^+(X))>t) dt- \intinfinity w^+(1- {\hat F^+_n}(t)) dt\right|>\epsilon\right)
\nonumber\\
& \leq P\left(HM\sup_{t\in \mathbb{R}} \left|(P(u^+(X)<t)-{\hat F^+_n}(t)\right|^\alpha>\epsilon\right)
\leq  e^{-n \frac{\epsilon ^{(2/\alpha)}} {2 H^2 M^2}}.\label{eq:dkw3}
\end{align}
\end{proof}

\subsection{Lipschitz continuous weights}
\label{sec:lipschitz-proofs}

 Setting $\alpha=\gamma=1$ in the proof of Proposition \ref{prop:lipschitz}, it is easy to see that the CPT-value \eqref{eq:cpt-mdp} is finite. 

Next, in order to prove the asymptotic convergence claim in Proposition \ref{prop:lipschitz}, we require the dominated convergence theorem in its generalized form, which is provided below.
\begin{theorem}{\textbf{\textit{(Generalized Dominated Convergence theorem)}}}
Let $\{f_n\}_{n=1}^\infty$ be a sequence of measurable functions on $E$ that converge pointwise a.e. on a measurable space $E$ to $f$.  Suppose there is a sequence $\{g_n\}$ of integrable functions on $E$ that converge pointwise a.e. on $E$ to $g$ such that $|f_n| \leq g_n$ for all $n \in \mathbb{N}$.  
If $\lim\limits_{n \rightarrow \infty}$ $\int_E$ $g_n$ = $\int_E$ $g$, then $\lim\limits_{n \rightarrow \infty}$ $\int_E$ $f_n$ = $\int_E$ $f$.
\end{theorem}

\begin{proof}
This is a standard result that can be found in any textbook on measure theory. For instance, see Theorem 2.3.11 in \cite{athreya2006measure}.
\end{proof}

\subsection*{Proof of Proposition \ref{prop:lipschitz}: Asymptotic convergence}
\begin{proof}
Notice the the following equivalence:
$$\sum_{i=1}^{n-1} u^+(X_{[i]}) (w^+(\frac{n-i}{n}) - w^+(\frac{n-i-1}{n})) =  \int_0^M w^+(1-\hat{F^+_n}(x)) dx, $$
and also,
$$\sum_{i=1}^{n-1} u^-(X_{[i]}) (w^-(\frac{i}{n}) - w^-(\frac{i+1}{n})) =  \int_0^M w^-(1-\hat{F^-_n}(x)) dx, $$
where $\hat{F^+_n}(x)$ and $\hat{F^-_n}(x)$ is the empirical distribution of $u^+(X)$
and $u^-(X)$.

Thus, the CPT estimator $\overline \C_n$ in Algorithm \ref{alg:holder-est} can be written equivalently as follows:
\begin{align}
\overline \C_n = \intinfinity w^+(1-{\hat F_n}^+(x))  dx - \intinfinity w^-(1-{\hat F_n}^-(x))  dx.
\label{eq:cpt-est-appendix}
\end{align}
We first prove the asymptotic convergence claim for the first integral in \eqref{eq:cpt-est-appendix}, i.e., we show
\begin{align}
\intinfinity w^+(1-{\hat F_n}^+(x))  dx \rightarrow \intinfinity w^+(P(u^+(X)>x) dx.\label{eq:3}
\end{align} 

Since $w^+$ is Lipschitz continuous with constant $L$, we have almost surely that
$w^{+}(1-\hat{F_n}(x)) \leq L (1-\hat{F_n}(x))$,  
for all $n$ and 
 $w^{+}((P(u^+(X)>x)) \leq L\cdot (P(u^+(X)>x)$, since $w^+(0)=0$.
 
Notice that the empirical distribution function 
${\hat F_n}^+(x)$ 
generates a Stieltjes measure which takes mass 
$1/n$ on each of the sample points $u^+(X_{i})$. 

We have
$$\intinfinity (P(u^+(X)>x))  dx = E(u^+(X))$$
and

\begin{equation}
\intinfinity (1-{\hat F_n}^+(x))  dx =\intinfinity \int_x^\infty d \hat{F_n}(t) dx.\label{eq:2}
\end{equation}
Since ${\hat F_n}^+(x)$ has bounded support on $\mathbb{R}$ $\forall n$, the integral in \eqref{eq:2} is finite.
Applying Fubini's theorem to the RHS of \eqref{eq:2}, we obtain
\begin{equation}
\intinfinity \int_x^\infty d \hat{F_n}(t) dx = \intinfinity \int_0^t dx d \hat{F_n}(t) = \intinfinity t d\hat{F_n}(t) = \frac{1}{n} \sum_{i=1}^n u^+(X_{[i]}),
 \end{equation}
 where $u^+(X_{[i]}), i=1,\ldots,n$ denote the order statistics, i.e., $u^+(X_{[1]}) \le \ldots \le u^+(X_{[n]})$.
 
Now, notice that 

$$
\frac{1}{n}
\sum_{i=1}^n u^+(X_{[i]})
=
\frac{1}{n}
\sum_{i=1}^n u^+(X_{[i]})
\overset{a.s}\longrightarrow 
E(u^+(X)),
$$
From the foregoing,
$$
\lim_{n\rightarrow \infty} \intinfinity L\cdot(1-\hat{F_n}(x)) dx
\overset{a.s} \longrightarrow
\intinfinity L \cdot(P(u^+(X)>x)) dx.$$
Hence, we have 
$$
\int_0^\infty w^{(+)}(1-\hat{F_n}(x)) dx \xrightarrow{a.s.} 
\int_0^\infty w^{(+)}(P(u^+(X))>x) dx.
$$
The claim in \eqref{eq:3} now follows by invoking the generalized dominated convergence theorem by setting $f_n = w^+(1-{\hat F_n}^+(x))$ and $g_n = L\cdot(1-\hat{F_n}(x))$, and noticing that $L\cdot(1-\hat{F_n}(x)) \xrightarrow{a.s.} L(P(u^+(X)>x))$ uniformly $\forall x$. The latter fact is implied by the Glivenko-Cantelli theorem (cf. Chapter 2 of \cite{wasserman2006}).

Following similar arguments, it is easy to show that 
$$
\intinfinity w^-(1-{\hat F_n}^-(x))  dx \rightarrow \intinfinity w^-(P(u^-(X))>x) dx.
$$
The final claim regarding the almost sure convergence of $\overline \C_n$ to $\C(X)$ now follows.
\end{proof}

\subsection*{Proof of Proposition \ref{prop:lipschitz}: Sample complexity}

\begin{proof}
Since $u^+(X)$ is bounded above by $M$ and $w^+$ is Lipschitz with constant $L$, we have
\begin{align*}
&\left|\intinfinity w^+(P(u^+(X))>x) dx- \intinfinity w^+(1- {\hat F_n}^+(x)) dx\right|
\\
= & \left|\int_0^M w^+(P(u^+(X))>x) dx- \int_0^M w^+(1- {\hat F_n}^+(x)) dx\right|
\\
\leq&
\left|\int_0^M L\cdot |P(u^+(X)<x)-{\hat F_n}^+(x)| dx\right|\\
\leq&
LM\sup_{x\in \mathbb{R}}\left|P(u^+(X)<x)-{\hat F_n}^+(x)\right|.
\end{align*}
Now, plugging in the DKW inequality, we obtain
\begin{align}
&
P\left(\left|\intinfinity w^+(P(u^+(X))>x) dx- \intinfinity w^+(1- {\hat F_n}^+(x)) dx\right|>\epsilon/2\right)
\nonumber\\
&
\leq
 P\left(LM\sup_{x\in \mathbb{R}} \left|(P(u^+(X)<x)-{\hat F_n}^+(x)\right|>\epsilon/2\right) \leq 2 e^{-n \frac{\epsilon^2}{2 L^2 M^2}}.\label{eq:dkw1}
\end{align}

Along similar lines, we obtain
\begin{align}
&
P\left(\left|\intinfinity w^-(P(u^-(X))>x) dx- \intinfinity w^-(1- {\hat F_n}^-(x)) dx\right|>\epsilon/2\right)
 \leq 2 e^{-n \frac{\epsilon^2}{2 L^2 M^2}}.\label{eq:dkw2}
\end{align}

Combining \eqref{eq:dkw1} and \eqref{eq:dkw2}, we obtain
\begin{align*}
P(|\overline \C_n - \C(X)|>\epsilon) 
&\le P\left(\left|\intinfinity w^+(P(u^+(X))>x) dx- \intinfinity w^+(1- {\hat F_n}^+(x)) dx\right|>\epsilon/2\right) \\
&+ 
P\left(\left|\intinfinity w^-(P(u^-(X))>x) dx- \intinfinity w^-(1- {\hat F_n}^-(x)) dx\right|>\epsilon/2\right)\\
&\le 4 e^{-n \frac{\epsilon^2}{2 L^2 M^2}}.
\end{align*} 
And the claim follows. 
\end{proof}

\subsection{Proofs for discrete valued $X$}
\label{sec:proofs-discrete}
Without loss of generality, assume $w^+=w^-=w$, and let
\begin{align}
\label{eq:hatFk}
\hat F_k = 
\begin{cases}
   \sum_{i=1}^k \hat p_k & \text{if   } k \leq l \\
   \sum_{i=k}^K \hat p_k & \text{if  }  k > l.
\end{cases}  
\end{align}

The following proposition gives the rate at which $\hat{F_k}$ converges to $F_k$.
\begin{proposition}
\label{prop:hoeffding-discrete}
Let $F_k$ and $\hat F_k$ be as defined in \eqref{eq:Fk}, \eqref{eq:hatFk}, Then, we have that, for every $\epsilon >0$, 
$$P(|\hat{F_k}-F_k| > \epsilon) \leq 2 e^{-2n \epsilon^2}.$$
\end{proposition}
\begin{proof}
We focus on the case when $k > l$, while the case of $k \leq l$ is proved in a similar fashion.
Notice that when $k>l$,  $\hat F_k =I_{(X_i \geq  x_k) }$. Since the random variables $X_i$ are independent of each other and  for each $i$, are bounded above by $1$, we can apply Hoeffding's inequality to obtain 
\begin{align*}
P(\left|\hat{F_k}- F_k \right| > \epsilon) & = P(\left| \frac{1}{n} \sum_{i=1}^n I_{\{X_i \geq
x_k\}} - \frac{1}{n} \sum_{i=1}^n E(I_{\{X_i \geq x_k\}}) \right| > \epsilon) \\ & = P(\left|
\sum_{i=1}^n I_{\{X_i \geq x_k\}} - \sum_{i=1}^n E(I_{\{X_i \geq x_k\}}) \right| > n\epsilon) \\ &
    \leq 2e^{-2n \epsilon^2}.
\end{align*}
\end{proof}

The proof of Proposition \ref{prop:sample-complexity-discrete} requires the following claim which gives the convergence rate under local Lipschitz weights. 

\begin{proposition}
\label{prop:discrete-first-term}
Under conditions of Proposition \ref{prop:sample-complexity-discrete}, with $F_k$ and $\hat F_k$ as defined in \eqref{eq:Fk} and \eqref{eq:hatFk},  we have
$$P(\left| \sum_{i=1}^K u_{k} w(\hat{F_k}) - \sum_{i=1}^K u_{k} w(F_k) \right| >\epsilon) < K\cdot (
e^{-\delta^2\cdot 2n} + e^{-\epsilon^2 2n/(KLA)^2}), \text{where}$$
\begin{align}
\label{eq:uplusminusxk}
u_k = 
\begin{cases}
   u^{-}(x_{k}) & \text{if   } k \leq l \\
   u^{+}(x_{k}) & \text{if  }  k > l.
\end{cases}  
\end{align} 
\end{proposition}

\begin{proof}
Observe that

\begin{align*}
P(\left| \sum_{k=1}^K u_k w(\hat{F_k}) - \sum_{k=1}^K u_k w(F_k) \right| >\epsilon) & = P (
\bigcup_{k=1}^K \left| u_k w(\hat{F_k}) -u_k w(F_k) \right| > \frac {\epsilon} {K}) \\ & \leq
    \sum_{k=1}^K P (\left| u_k w(\hat{F_k}) -u_k w(F_k) \right| > \frac {\epsilon} {K})
\end{align*}
Notice that $\forall k =1,....K$
$[{p_k}- \delta, {p_k}+\delta)$,
the function $w$ is locally Lipschitz with common constant $L$.
Therefore, for each k, we can decompose the probability as 
\begin{align*}
& P (\left| u_k w(\hat{F_k}) -u_k w(F_k) \right| > \frac {\epsilon} {K}) \\ & = P ( [ \left| F_k -
\hat{F_k} \right| >\delta ] \bigcap [ \left| u_k w(\hat{F_k}) -u_k w(F_k) \right| ] > \frac
{\epsilon} {K}) + P ( [ \left| F_k - \hat{F_k} \right| \leq\delta ] \bigcap [ \left| u_k
    w(\hat{F_k}) -u_k w(F_k) \right| ] > \frac {\epsilon} {K}) \\ & \leq P ( \left| F_k - \hat{F_k}
    \right| >\delta) + P ( [ \left| F_k - \hat{F_k} \right| \leq\delta ] \bigcap [ \left| u_k
    w(\hat{F_k}) -u_k w(F_k) \right| ] > \frac {\epsilon} {K}).
\end{align*}
 
According to the property of locally Lipschitz continuous,
we have
\begin{align*}
& P ( [ \left| F_k - \hat{F_k} \right| \leq\delta ] \bigcap [ \left| u_k w(\hat{F_k}) -u_k w(F_k)
\right| ] > \frac {\epsilon} {K}) \\ & \leq P(u_k L \left| F_k - \hat{F_k} \right| > \frac
    {\epsilon} {K}) \leq e^ {-\epsilon\cdot 2n /(K L u_k)^2} \leq e^ {-\epsilon\cdot 2n /(K L A)^2},
     \forall k.
\end{align*}
And similarly,
\begin{align*}
P(\left| F_k - \hat{F_k} \right| > \delta) & \leq e^{-\delta^2 /2n},   \forall
    k.
\end{align*}
And as a result,
\begin{align*}
 P(\left| \sum_{k=1}^K u_k w(\hat{F_k}) - \sum_{k=1}^K u_k w(F_k) \right| >\epsilon) & \leq
\sum_{k=1}^K P (\left| u_k w(\hat{F_k}) -u_k w(F_k) \right| > \frac {\epsilon} {K}) \\ & \leq
             \sum_{k=1}^K \left( e^{-\delta^2\cdot 2n} + e^{-\epsilon^2 \cdot 2n/ (KLA)^2} \right) \\ & =K\cdot
    (e^{-\delta^2\cdot 2n} + e^{-\epsilon^2 \cdot 2n/ (KLA)^2})
\end{align*}
\end{proof}

\subsection*{Proof of Proposition \ref{prop:sample-complexity-discrete}}
\begin{proof}
With $u_k$ as defined in \eqref{eq:uplusminusxk}, we need to prove that
\begin{align}
P(\left|\sum_{i=1}^K u_{k} \cdot(w(\hat{F_k})- w(\hat F_{k+1}) )
-  
\sum_{i=1}^K u_{k} \cdot(w(F_k)- w(F_{k+1}) )
\right| \leq \epsilon) > 1-\rho
\text{      ,     } \forall n> \frac{\ln(\frac{4K}{A})} { M}, 
\label{eq:dw}
\end{align}
where $w$ is Locally Lipschitz continuous with constants $L_1,....L_K$ at the points $F_1,....F_K$.
From a parallel argument to that in the proof of Proposition \ref{prop:discrete-first-term}, it is easy to infer that
$$P(\left| \sum_{i=1}^K u_k w(\hat F_{k+1}) - \sum_{i=1}^K u_k w(F_{k+1}) \right| >\epsilon) <
K\cdot ( e^{-\delta^2\cdot 2n} + e^{-\epsilon^2 2n/(KLA)^2})
$$
Hence,
\begin{align*}
& P(\left|\sum_{i=1}^K u_k \cdot(w(\hat{F_k})- w(\hat F_{k+1}) ) -  \sum_{i=1}^K u_k \cdot(w(F_k)-
w(F_{k+1}) ) \right| > \epsilon) \\ 
\leq & \quad P(\left|\sum_{i=1}^K u_k \cdot(w(\hat{F_k})) -
    \sum_{i=1}^K u_k \cdot(w(F_k)) \right| > \epsilon/2) \\
		&+ P(\left|\sum_{i=1}^K u_k
    \cdot(w(\hat F_{k+1})) -  \sum_{i=1}^K u_k \cdot(w(F_{k+1})) \right| > \epsilon/2) \\ \leq & \quad 2K
    (e^{-\delta^2\cdot 2n} + e^{-\epsilon^2 2n/(KLA)^2})
\end{align*}
The claim in \eqref{eq:dw} now follows.
\end{proof}

\section{Proofs for CPT-SPSA-G}
\label{appendix:1spsa}

To prove the main result in Theorem \ref{thm:1spsa-conv}, we first show, in the following lemma, that the gradient estimate using SPSA is only an order $O(\delta_n^2)$ term away from the true gradient. The proof differs from the corresponding claim for regular SPSA (see Lemma 1 in \cite{spall}) since we have a non-zero bias in the function evaluations, while the regular SPSA assumes the noise is zero-mean. Following this lemma, we complete the proof of Theorem \ref{thm:1spsa-conv} by invoking the well-known Kushner-Clark lemma \cite{kushner-clark}.

\todop{The bias control is with high prob and so probably all conv claims shoud be with high prob. Probably there is a simpler way out, but i dont know (yet)}
\begin{lemma}
\label{lemma:1spsa-bias}
Let $\F_n = \sigma(\theta_m,m\le n)$, $n\ge 1$.
Then, for any $i = 1,\ldots,d$, we have almost surely,  
\begin{align}
\left| \E\left[\left.\dfrac{\overline \C_n^{\theta_n+\delta_n \Delta_n} -\overline \C_n^{\theta_n-\delta_n \Delta_n}}{2 \delta_n \Delta_n^{i}}\right| \F_n \right] - \nabla_i \C(X^{\theta_n}) \right| \rightarrow 0 \text{ as } n\rightarrow\infty.
\end{align} 
\end{lemma}
\begin{proof}
Recall that the CPT-value estimation scheme is biased, i.e., providing samples with policy $\theta$, we obtain its CPT-value estimate as $V^{\theta}(x_0) + \epsilon^\theta$. Here $\epsilon^\theta$ denotes the bias. 

We claim
\begin{align}
\quad\E\left[\dfrac{\overline \C_n^{\theta_n+\delta_n \Delta_n} -\overline \C_n^{\theta_n-\delta_n \Delta_n}}{2 \delta_n \Delta_n^{i}} \left.\right| \F_n\right] 
= \quad&\E\left[\dfrac{\C(X^{\theta_n + \delta_n \Delta_n}) - \C(X^{\theta_n - \delta_n \Delta_n})}{2 \delta_n \Delta_n^{i}} \left.\right| \F_n\right] + \E\left[ \eta_n \mid \F_n\right],
\end{align}
where $\eta_n = \left(\dfrac{\epsilon^{\theta_n +\delta_n\Delta} - \epsilon^{\theta_n-\delta_n\Delta}}{2\delta_n\Delta_n^{i}}\right)$
 is the bias arising out of the empirical distribution based CPT-value estimation scheme.
From Proposition \ref{prop:holder-dkw} and the fact that $\frac{1}{m_n^{\alpha/2} \delta_n} \rightarrow 0$ by assumption (A3), we have that
$\eta_n$ goes to zero asymptotically. In other words,
\begin{align}
\quad\E\left[\dfrac{\overline \C_n^{\theta_n+\delta_n \Delta_n} -\overline \C_n^{\theta_n-\delta_n \Delta_n}}{2 \delta_n \Delta_n^{i}} \left.\right| \F_n\right] 
&\xrightarrow{n\rightarrow\infty}  \E\left[\dfrac{\C(X^{\theta_n + \delta_n \Delta_n}) -\C(X^{\theta_n - \delta_n \Delta_n})}{2 \delta_n \Delta_n^{i}} \left.\right| \F_n\right].  \label{eq:l1}
\end{align}
We now analyse the RHS of \eqref{eq:l1}.
By using suitable Taylor's expansions,
\begin{align*}
\C(X^{\theta_n + \delta_n \Delta_n}) = \C(X^{\theta_n}) + \delta_n \Delta_n\tr \nabla \C(X^{\theta_n}) + \frac{\delta^2}{2} \Delta_n\tr \nabla^2 \C(X^{\theta_n})\Delta_n + O(\delta_n^3), \\
\C(X^{\theta_n - \delta_n \Delta_n}) = \C(X^{\theta_n}) - \delta_n \Delta_n\tr \nabla \C(X^{\theta_n}) + \frac{\delta^2}{2} \Delta_n\tr \nabla^2 \C(X^{\theta_n})\Delta_n + O(\delta_n^3).\end{align*}
From the above, it is easy to see that 
\begin{align*}
\dfrac{\C(X^{\theta_n + \delta_n \Delta_n}) - \C(X^{\theta_n - \delta_n \Delta_n})}{2 \delta_n \Delta_n^i}
- \nabla_i \C(X^{\theta_n})
=\underbrace{\sum_{j=1,j\not=i}^{N} \frac{\Delta_n^j}{\Delta_n^i}\nabla_j \C(X^{\theta_n})}_{(I)} + O(\delta_n^2).
\end{align*}
Taking conditional expectation on both sides, we obtain
\begin{align}
\E\left[\dfrac{\C(X^{\theta_n + \delta_n \Delta_n}) - \C(X^{\theta_n - \delta_n \Delta_n})}{2 \delta_n \Delta_n^i} \left.\right| \F_n\right] 
=& \nabla_i \C(X^{\theta_n}) + \E\left[\sum_{j=1,j\not=i}^{N} \frac{\Delta_n^j}{\Delta_n^i}\right]\nabla_j \C(X^{\theta_n}) + O(\delta_n^2)\nonumber\\
= & \nabla_i \C(X^{\theta_n}) + O(\delta_n^2).\label{eq:l2}
\end{align}
The first equality above follows from the fact that $\Delta_n$ is distributed according to a $d$-dimensional vector of Rademacher random variables and is independent of $\F_n$. The second inequality follows by observing that $\Delta_n^i$ is independent of $\Delta_n^j$, for any $i,j =1,\ldots,d$, $j\ne i$. 

The claim follows by using the fact that $\delta_n \rightarrow 0$ as $n\rightarrow \infty$.
\end{proof}

\subsection*{Proof of Theorem \ref{thm:1spsa-conv}}

\begin{proof}

We first rewrite the update rule \eqref{eq:theta-update} as follows: For $i=1,\ldots,d$,
\begin{align}
\theta^{i}_{n+1}  =  \theta^{i}_n +  \gamma_n(\nabla_{i} \C(X^{\theta_n}) + \beta_n + \xi_n), 
\label{eq:1spsa-equiv}
\end{align}
where 
\begin{align*}
\beta_n \quad= &\quad \E\left(\dfrac{(\overline \C_n^{\theta_n + \delta_n \Delta_n} -\overline \C_n^{\theta_n - \delta_n \Delta_n})}{2 \delta_n \Delta_n^{i}} \mid \F_n \right) - \nabla \C(X^{\theta_n}), \text{ and}\\
\xi_n \quad = & \quad\left(\dfrac{\overline \C_n^{\theta_n + \delta_n \Delta_n} -\overline \C_n^{\theta_n - \delta_n \Delta_n}}{2 \delta_n\Delta_n^{i}}\right)  - \E\left(\dfrac{(\overline \C_n^{\theta_n + \delta_n \Delta_n} -\overline \C_n^{\theta_n - \delta_n \Delta_n})}{2 \delta_n \Delta_n^{i}}\mid \F_n \right).
\end{align*}
In the above, $\beta_n$ is the bias in the gradient estimate due to SPSA and $\xi_n$ is a martingale difference sequence..

Convergence of \eqref{eq:1spsa-equiv} can be inferred from Theorem 5.3.1 on pp. 191-196 of \cite{kushner-clark}, provided we verify the necessary assumptions given as (B1)-(B5) below:
\begin{enumerate}[\bfseries (B1)]
\item $\nabla \C(X^{\theta})$ is a continuous $\R^{d}$-valued function.
\todop{Don't know how this gets verfied in our setting. Help!}
\item  The sequence $\beta_n,n\geq 0$ is a bounded random sequence with
$\beta_n \rightarrow 0$ almost surely as $n\rightarrow \infty$.
\item The step-sizes $\gamma_n,n\geq 0$ satisfy
$  \gamma_n\rightarrow 0 \mbox{ as }n\rightarrow\infty \text{ and } \sum_n \gamma_n=\infty.$
\item $\{\xi_n, n\ge 0\}$ is a sequence such that for any $\epsilon>0$,
\[ \lim_{n\rightarrow\infty} P\left( \sup_{m\geq n}  \left\|
\sum_{k=n}^{m} \gamma_k \xi_k\right\| \geq \epsilon \right) = 0. \]
\item There exists a compact subset $K$ which is the set of asymptotically stable equilibrium points for the following ODE:
\begin{align}
\dot\theta^{i}_t = \check\Gamma_{i}\left(- \nabla \C(X^{\theta^{i}_t})\right), \text{ for } i=1,\dots,d,\label{eq:pi-ode}
\end{align}
\end{enumerate} 

In the following, we verify the above assumptions for the recursion \eqref{eq:theta-update}:
\begin{itemize}
\item (B1) holds by assumption in our setting.

\item Lemma \ref{lemma:1spsa-bias} above establishes that the bias $\beta_n$ is $O(\delta_n^2)$ and since $\delta_n \rightarrow 0$ as $n\rightarrow \infty$, it is easy to see that (B2) is satisfied for $\beta_n$. 

\item (B3) holds by assumption (A3).

\item We verify (B4) using arguments similar to those used in \cite{spall} for the classic SPSA algorithm:\\
We first recall Doob's martingale inequality (see (2.1.7) on pp. 27 of \cite{kushner-clark}):
\begin{align}
P\left( \sup_{m\geq 0}   \left\|W_l\right\| \geq \epsilon \right) \le \dfrac{1}{\epsilon^2} \lim_{l\rightarrow \infty} \E \left\|W_l\right\|^2. 
\end{align}
Applying the above inequality to the martingale sequence $\{W_l\}$, where  $W_l := \sum_{n=0}^{l-1} \gamma_n \eta_n$, $l\ge 1$, we obtain
\begin{align}
P\left( \sup_{l\geq k}   \left\|\sum_{n=k}^{l} \gamma_n \xi_n\right\| \geq \epsilon \right) \le \dfrac{1}{\epsilon^2} \E \left\|
\sum_{n=k}^{\infty} \gamma_n \xi_n\right\|^2 = \dfrac{1}{\epsilon^2} \sum_{n=k}^{\infty} \gamma_n^2 \E\left\| \eta_n\right\|^2. \label{eq:b4}
\end{align}
The last equality above follows by observing that, for $m < n$, $\E(\xi_m \xi_n) = \E(\xi_m \E(\xi_n\mid \F_n))=0$.
We now bound $\E\left\| \xi_n\right\|^2$ as follows:
\begin{align}
\E\left\| \xi_n\right\|^2\le &\E\left(\dfrac{\overline \C_n^{\theta_n + \delta_n \Delta_n} -\overline \C_n^{\theta_n - \delta_n \Delta_n}}{2 \delta_n \Delta_n^i}\right)^2 \label{eq:mi}\\
\le &\left(\left(\E\left(\dfrac{\overline \C_n^{\theta_n + \delta_n \Delta_n}}{2 \delta_n \Delta_n^i}\right)^2\right)^{1/2}
+ \left(\E\left(\dfrac{\overline \C_n^{\theta_n - \delta_n \Delta_n}}{2 \delta_n \Delta_n^i}\right)^2\right)^{1/2}\right)^2 \label{eq:minko}\\
\le &\frac1{4\delta_n^2} \left[ \E\left(\frac1{(\Delta_n^i)^{2+2\alpha_1}}\right) \right]^{\frac{1}{1+\alpha1}} \nonumber\\
& \times \left(\left[\E\left[ (\overline \C_n^{\theta_n + \delta_n \Delta_n})\right]^{2+2\alpha_2}\right]^{\frac{1}{1+\alpha_2}} +
\left[\E\left[ (\overline \C_n^{\theta_n - \delta_n \Delta_n})\right]^{2+2\alpha_2}\right]^{\frac{1}{1+\alpha_2}}\right)\label{eq:holder}\\
\le &\frac1{4\delta_n^2} \left(\left[\E\left[ (\overline \C_n^{\theta_n + \delta_n \Delta_n})\right]^{2+2\alpha_2}\right]^{\frac{1}{1+\alpha_2}} +
\left[\E\left[ (\overline \C_n^{\theta_n - \delta_n \Delta_n})\right]^{2+2\alpha_2}\right]^{\frac{1}{1+\alpha_2}}\right) \label{eq:h2}\\
\le & \frac{C}{\delta_n^2}, \text{ for some } C< \infty. \label{eq:h3}
\end{align}
The inequality in \eqref{eq:mi} uses the fact that, for any random variable $X$, $\E\left\|X -  E[X\mid\F_n]\right\|^2 \le \E X^2$. The inequality in \eqref{eq:minko} follows by the fact that $\E (X+Y)^2 \le \left( (\E X^2)^{1/2} + (\E Y^2)^{1/2}\right)^2$.
The inequality in \eqref{eq:holder} uses Holder's inequality, with $\alpha_1, \alpha_2>0$ satisfying $\frac{1}{1+\alpha_1} + \frac{1}{1+\alpha_2}=1$. 
The equality in \eqref{eq:h2} above follows owing to the fact that $\E\left(\frac1{(\Delta_n^i)^{2+2\alpha_1}}\right)  = 1$ as $\Delta_n^i$ is Rademacher. 
The inequality in \eqref{eq:h3} follows by using the fact that
$\C(D^\theta)$ is bounded for any policy $\theta$ and the bias $\epsilon^\theta$ is bounded by Proposition \ref{prop:holder-dkw}.
\todop{Need to update the above arguments for the general case of $X^\theta$, with $\theta$ in a compact set}

Thus, $\E\left\| \xi_n\right\|^2 \le \frac{C}{\delta_n^2}$ for some $C<\infty$. Plugging this in \eqref{eq:b4}, we obtain
\begin{align*}
 \lim_{k\rightarrow\infty} P\left( \sup_{l\geq k}   \left\|\sum_{n=k}^{l} \gamma_n \xi_n\right\|\geq \epsilon \right) \le \dfrac{d C}{\epsilon^2} \lim_{k\rightarrow\infty} \sum_{n=k}^{\infty}  \frac{\gamma_n^2}{\delta_n^2} =0.
\end{align*}
The equality above follows from (A3) in the main paper.
\item Observe that $\C(X^{\theta})$ serves as a strict Lyapunov function for the ODE \eqref{eq:pi-ode}. This can be seen as follows:
$$ \dfrac{d \C(X^{\theta})}{dt} = \nabla \C(X^{\theta}) \dot \theta = \nabla \C(X^{\theta}) \check\Gamma \left(-\nabla \C(X^{\theta}\right) < 0.$$
Hence, the set $\K = \{\theta \mid \check\Gamma_{i} \left(-\nabla \C(X^{\theta})\right)=0, \forall i=1,\ldots,d\}$ serves as the asymptotically stable attractor for the ODE \eqref{eq:pi-ode}.
\end{itemize} 
The claim follows from the Kushner-Clark lemma.
\end{proof}
\section{Newton algorithm for CPT-value optimization (CPT-SPSA-N)}
\label{sec:2spsa}
\subsection{Need for second-order methods}
While stochastic gradient descent methods are useful in minimizing the CPT-value given biased estimates, they are sensitive to the choice of the step-size sequence $\{\gamma_n\}$.  In particular, for a step-size choice $\gamma_n = \gamma_0/n$, if $a_0$ is not chosen to be greater than $1/3 \lambda_{min}(\nabla^2 \C(X^{\theta^*}))$, then the optimum rate of convergence is not achieved, where $\lambda_{\min}$ denotes the minimum eigenvalue, while $\theta^*\in \K$ (see Theorem \ref{thm:1spsa-conv}). A standard approach to overcome this step-size dependency is to use iterate averaging, suggested independently by Polyak \cite{polyak1992acceleration} and Ruppert \cite{ruppert1991stochastic}. The idea is to use larger step-sizes $\gamma_n = 1/n^\varsigma$, where $\varsigma \in (1/2,1)$, and then combine it with averaging of the iterates. However, it is well known  that iterate averaging is optimal only in an asymptotic sense, while finite-time bounds show that the initial condition is not forgotten sub-exponentially fast (see 
Theorem 2.2 in \cite{fathi2013transport}). Thus, it is optimal to average iterates only 
after a sufficient number of iterations have passed and all the iterates are very close to the optimum. However, the latter situation serves as a stopping condition in practice.

An alternative approach is to employ step-sizes of the form $\gamma_n = (a_0/n) M_n$, where $M_n$ converges to $\left(\nabla^2 \C(X^{\theta^*})\right)^{-1}$, i.e., the inverse of the Hessian of the CPT-value at the optimum $\theta^*$. Such a scheme gets rid of the step-size dependency (one can set $a_0=1$) and still obtains optimal convergence rates. This is the motivation behind having a second-order optimization scheme.

\subsection{Gradient and Hessian estimation}
We estimate the Hessian of the CPT-value function using the scheme suggested by \cite{bhatnagar2015simultaneous}. As in the first-order method, we use Rademacher random variables to simultaneously perturb all the coordinates. However, in this case, we require three system trajectories with corresponding  parameters $\theta_n+\delta_n(\Delta_n+\widehat\Delta_n)$, $\theta_n-\delta_n(\Delta_n+\widehat\Delta_n)$ and $\theta_n$, where $\{\Delta_n^i, \widehat\Delta_n^i, i=1,\ldots,d\}$ are i.i.d. Rademacher and independent of $\theta_0,\ldots,\theta_n$. Using the CPT-value estimates for the aforementioned  parameters, we estimate the Hessian and the gradient of the CPT-value function as follows: For $i,j=1,\ldots,d$, set
\begin{align*}
&\widehat \nabla_{i} \C(X_n^{\theta_n})=\dfrac{\overline \C_n^{\theta_n+\delta_n(\Delta_n+\widehat\Delta_n)} - \overline \C_n^{\theta_n-\delta_n(\Delta_n+\widehat\Delta_n)}}{2\delta_n \Delta_n^{i}},\\ 
&\widehat H_n^{i,j}=\dfrac{\overline \C_n^{\theta_n+\delta_n(\Delta_n+\widehat\Delta_n} + \overline \C_n^{\theta_n-\delta_n(\Delta_n+\widehat\Delta_n} - 2\overline \C_n^{\theta_n}}{\delta_n^2 \Delta_n^{i}\widehat\Delta_n^{j}}.
\end{align*}
Notice that the above estimates require three samples, while the second-order SPSA algorithm proposed first in \cite{spall2000adaptive} required four.
Both the gradient estimate $\widehat \nabla \C(X_n^{\theta_n}) = [\widehat \nabla_i \C(X_n^{\theta_n})], i=1,\ldots,d,$ and the Hessian estimate $\widehat{H_n} = [\widehat H_n^{i,j}], i,j=1,\ldots,d,$ can be shown to be an $O(\delta_n^2)$ term away from the true gradient $\nabla \C(X^\theta_n)$ and Hessian $\nabla^2  \C(X^\theta_n)$, respectively (see Lemmas \ref{lemma:2spsa-bias}--\ref{lemma:2spsa-grad}).

\algblock{PEvalPrimeDouble}{EndPEvalPrimeDouble}
\algnewcommand\algorithmicPEvalPrimeDouble{\textbf{\em CPT-value Estimation (Trajectory 3)}}
 \algnewcommand\algorithmicendPEvalPrimeDouble{}
\algrenewtext{PEvalPrimeDouble}[1]{\algorithmicPEvalPrimeDouble\ #1}
\algrenewtext{EndPEvalPrimeDouble}{\algorithmicendPEvalPrimeDouble}
\algtext*{EndPEvalPrimeDouble}

\algblock{PImpNewton}{EndPImpNewton}
\algnewcommand\algorithmicPImpNewton{\textbf{\em Newton step}}
 \algnewcommand\algorithmicendPImpNewton{}
\algrenewtext{PImpNewton}[1]{\algorithmicPImpNewton\ #1}
\algrenewtext{EndPImpNewton}{\algorithmicendPImpNewton}

\algtext*{EndPImpNewton}


\begin{algorithm}[t]
\begin{algorithmic}
\State {\bf Input:} 
initial parameter $\theta_0 \in \Theta$ where $\Theta$ is a compact and convex subset of $\R^d$, perturbation constants $\delta_n>0$, sample sizes $\{m_n\}$, step-sizes $\{\gamma_n, \xi_n\}$, operator $\Gamma: \R^d \rightarrow \Theta$.
\For{$n = 0,1,2,\ldots$}	
	\State Generate $\{\Delta_n^{i}, \widehat\Delta_n^{i}, i=1,\ldots,d\}$ using Rademacher distribution, independent of $\{\Delta_m, \widehat \Delta_m, m=0,1,\ldots,n-1\}$.
	\PEval
	    \State Simulate $m_n$ samples  using parameter $(\theta_n+\delta_n (\Delta_n + \hat \Delta_n))$.
	    \State Obtain CPT-value estimate $\overline \C_n^{\theta_n+\delta_n (\Delta_n+\hat \Delta_n)}$.
	    \EndPEval
	    \PEvalPrime
  	    \State Simulate $m_n$ samples using parameter $(\theta_n-\delta_n (\Delta_n + \hat \Delta_n))$.
	    \State Obtain CPT-value estimate $\overline \C_n^{\theta_n-\delta_n (\Delta_n+\hat\Delta_n)}$.
	    \EndPEvalPrime
	    	    \PEvalPrimeDouble
  	    \State Simulate $m_n$ samples using parameter $\theta_n$.
	    \State Obtain CPT-value estimate $\overline \C_n^{\theta_n}$ using Algorithm \ref{alg:holder-est}.
	    \EndPEvalPrimeDouble
	    \PImpNewton
		\State Update the parameter and Hessian according to \eqref{eq:2spsa}--\eqref{eq:2spsa-H}.
		\EndPImpNewton
\EndFor
\State {\bf Return} $\theta_n.$
\end{algorithmic}
\caption{Structure of CPT-SPSA-N algorithm.}
\label{alg:structure-2}
\end{algorithm}

\subsection{Update rule}
We update the parameter incrementally using a Newton decrement as follows: For $i=1,\ldots,d$,
\begin{align}
\label{eq:2spsa}
\theta^{i}_{n+1} =& \Gamma_{i}\left(\theta^{i}_n + \gamma_n \sum_{j=1}^{d} M_n^{i,j} \widehat \nabla_{j} \C(X^\theta_n)\right), \\
\overline H_n = & (1-\xi_n) \overline H_{n-1} + \xi_n \widehat H_n,\label{eq:2spsa-H}
\end{align}
where $\xi_n$ is a step-size sequence that satisfies 
$\sum_{n} \xi_n = \infty, \sum_n \xi_n^2 < \infty$ and $\frac{\gamma_n}{\xi_n}\rightarrow 0$ as $n\rightarrow \infty$. These conditions on $\xi_n$ ensure that the updates to $\overline H_n$ proceed on a timescale that is faster than that of $\theta_n$ in \eqref{eq:2spsa} - see Chapter 6 of \cite{borkar2008stochastic}.
Further, $\Gamma$ is a projection operator as in CPT-SPSA-G and  $M_n = [M_n^{i,j}] = \Upsilon(\overline H_n)^{-1}$.
Notice that we invert $\overline H_n$ in each iteration, and to ensure that this inversion is feasible (so that the $\theta$-recursion descends), we project $\overline H_n$ onto the set of positive definite matrices using the operator $\Upsilon$. The operator has to be such that asymptotically $\Upsilon(\overline H_n)$ should be the same as $\overline H_n$ (since the latter would converge to the true Hessian), while ensuring inversion is feasible in the initial iterations.  The assumption below makes these requirements precise.\\[1ex]
\textbf{Assumption (A4).}  For any $\{A_n\}$ and $\{B_n\}$,
${\displaystyle \lim_{n\rightarrow \infty} \left\| A_n-B_n \right\|}= 0 \Rightarrow {\displaystyle \lim_{n\rightarrow \infty} \parallel \Upsilon(A_n)- \Upsilon(B_n) \parallel}= 0$. Further, for any $\{C_n\}$  with
${\displaystyle \sup_n \parallel C_n\parallel}<\infty$,
${\displaystyle \sup_n \left(\parallel \Upsilon(C_n)\parallel + \parallel \{\Upsilon(C_n)\}^{-1} \parallel\right) < \infty}$.
\\[0.5ex]
A simple way to ensure the above is to have $\Upsilon(\cdot)$ as a diagonal matrix and then add a positive scalar $\delta_n$ to the diagonal elements so as to ensure invertibility  - see \cite{gill1981practical}, \cite{spall2000adaptive} for a similar operator.

Algorithm \ref{alg:structure-2} presents the pseudocode.  

\subsection{Convergence result}
\begin{theorem}
\label{thm:2spsa}
Assume (A1)-(A4). 
Consider the ODE: 
$$
\dot\theta^{i}_t = \check\Gamma_{i}\left( - \Upsilon(\nabla^2 \C(X^{\theta_t}))^{-1} \nabla \C(X^{\theta^{i}_t}) \right), \text { for }i=1,\dots,d,$$
where 
$\bar\Gamma_{i}$ is as defined in Theorem \ref{thm:1spsa-conv}. Let $\K = \{\theta \in \Theta \mid
\nabla \C(X^{\theta^{i}})  \check\Gamma_{i}\left(-\Upsilon(\nabla^2 \C(X^{\theta}))^{-1} \nabla \C(X^{\theta^{i}})\right)
=0, \forall i=1,\ldots,d\}$. Then, for $\theta_n$ governed by \eqref{eq:2spsa}, 
we have
$$\theta_n \rightarrow \K  \text{~~ a.s. as } n\rightarrow \infty.$$ 
\end{theorem}
\begin{proof}
Before proving Theorem \ref{thm:2spsa}, we bound the bias in the SPSA based estimate of the Hessian in the following lemma.
\begin{lemma}
\label{lemma:2spsa-bias}
For any $i, j= 1,\ldots,d$, we have almost surely,  
\begin{align}
    \left| \E\left[\left.\dfrac{\overline \C_n^{\theta_n+\delta_n(\Delta_n+\widehat\Delta_n)} + \overline \C_n^{\theta_n-\delta_n(\Delta_n+\widehat\Delta_n)} - 2 \overline \C_n^{\theta_n}}{\delta_n^2 \Delta_n^{i}\widehat\Delta_n^j}\right| \F_n \right] - \nabla^2_{i,j} \C(X^{\theta_n}) \right| \rightarrow 0 \text{ as } n\rightarrow\infty.
\end{align} 
\end{lemma}
\begin{proof}
As in the proof of Lemma \ref{lemma:1spsa-bias}, we can ignore the bias from the CPT-value estimation scheme and conclude that
\begin{align}
    \quad&\E\left[\dfrac{\overline \C_n^{\theta_n+\delta_n(\Delta_n+\widehat\Delta_n)} + \overline \C_n^{\theta_n-\delta_n(\Delta_n+\widehat\Delta_n)} - 2\overline \C_n^{\theta_n}}{\delta_n^2 \Delta_n^i\widehat\Delta_n^j} \left.\right| \F_n\right] \nonumber\\
     &\xrightarrow{n\rightarrow\infty}  \E\left[\dfrac{\C(X^{\theta_n+\delta_n(\Delta_n+\widehat\Delta_n)}) + \C(X^{\theta_n-\delta_n(\Delta_n+\widehat\Delta_n)}) - 2\C(X^{\theta_n})}{\delta_n^2 \Delta_n^i\widehat\Delta_n^j} \left.\right| \F_n\right].  \label{eq:l21}
\end{align}
Now, the RHS of \eqref{eq:l21} approximates the true gradient with only an $O(\delta_n^2)$ error; this can be inferred using arguments similar to those used in the proof of Proposition 4.2 of \cite{bhatnagar2015simultaneous}. We provide the proof here for the sake of completeness.
Using Taylor's expansion as in Lemma \ref{lemma:1spsa-bias}, we obtain
\begin{align*}
&\dfrac{\C(X^{\theta_n+\delta_n(\Delta_n+\widehat\Delta_n)}) + \C(X^{\theta_n-\delta_n(\Delta_n+\widehat\Delta_n)}) - 2\C(X^{\theta_n})}{\delta_n^2 \Delta_n^i\widehat\Delta_n^j}\\
=&  \frac{(\Delta_n+\hat{\Delta_n})\tr \nabla^2 \C(X^{\theta_n})(\Delta_n
+\hat{\Delta_n})}{\triangle_i(n)\hat{\triangle}_j(n)}
+ O(\delta_n^2) \\
=& \sum_{l=1}^{d}\sum_{m=1}^{d} \frac{\Delta_n^l\nabla^2_{l,m}\C(X^{\theta_n})\Delta_n^m}{\Delta_n^i
\hat{\Delta}_n^j} + 2\sum_{l=1}^{d}\sum_{m=1}^{d} \frac{\Delta_n^l\nabla^2_{l,m}
\C(X^{\theta_n})\hat{\Delta}_n^m}{\Delta_n^i
\hat{\Delta}_n^j}+ \sum_{l=1}^{d}\sum_{m=1}^{d} \frac{\hat{\Delta}_n^l\nabla^2_{l,m}\C(X^{\theta_n})\hat{\Delta}_n^m}{\Delta_n^i
\hat{\Delta}_n^j} + O(\delta_n^2).
\end{align*}
Taking conditional expectation, we observe that the first and last term above become zero, while the second term becomes $\nabla^2_{ij}
\C(X^{\theta_n})$. The claim follows by using the fact that $\delta_n \rightarrow 0$ as $n\rightarrow \infty$.
\end{proof}

\begin{lemma}
\label{lemma:2spsa-grad}
For any $i = 1,\ldots,d$, we have almost surely,  
\begin{align}
\left| \E\left[\left.\dfrac{\overline \C_n^{\theta_n + \delta_n (\Delta_n + \hat \Delta_n)} -\overline \C_n^{\theta_n - \delta_n (\Delta_n + \hat \Delta_n)}}{2 \delta_n \Delta_n^i}\right| \F_n \right] - \nabla_i \C(X^{\theta_n}) \right| \rightarrow 0 \text{ as } n\rightarrow\infty.
\end{align} 
\end{lemma}
\begin{proof}
As in the proof of Lemma \ref{lemma:1spsa-bias}, we can ignore the bias from the CPT-value estimation scheme and conclude that
\begin{align*}
\quad&\E\left[\dfrac{\overline \C_n^{\theta_n + \delta_n (\Delta_n+\widehat\Delta_n)} -\overline \C_n^{\theta_n - \delta_n (\Delta_n+\widehat\Delta_n)}}{2 \delta_n \Delta_n^i} \left.\right| \F_n\right] 
\xrightarrow{n\rightarrow\infty}  \E\left[\dfrac{\C(X^{\theta_n + \delta_n \Delta_n}) -\C(X^{\theta_n - \delta_n \Delta_n})}{2 \delta_n \Delta_n^i} \left.\right| \F_n\right].  
\end{align*}
The rest of the proof amounts to showing that the RHS of the above approximates the true gradient with an $O(\delta_n^2)$ correcting term; this can be done in a similar manner as the proof of Lemma \ref{lemma:1spsa-bias}.
\end{proof}

\subsection*{Proof of Theorem \ref{thm:2spsa}}

Before we prove Theorem \ref{thm:2spsa}, we show that the Hessian recursion \eqref{eq:2spsa-H} converges to the true Hessian, for any policy $\theta$.

\begin{lemma}
\label{lemma:h-est}
For any $i, j= 1,\ldots,d$, we have almost surely,  
$$\left \| H^{i, j}_n - \nabla^2_{i,j} \C(X^{\theta_n}) \right \| \rightarrow 0,
\text{ and }\left \| \Upsilon(\overline H_n)^{-1} - \Upsilon(\nabla^2_{i,j} \C(X^{\theta_n}))^{-1} \right \| \rightarrow 0.
$$
\end{lemma}
\begin{proof}
 Follows in a similar manner as in the proofs of Lemmas 7.10 and 7.11 of \cite{Bhatnagar13SR}.
\end{proof}

\begin{proof}\textbf{\textit{(Theorem \ref{thm:2spsa})}}
The proof follows in a similar manner as the proof of Theorem 7.1 in \cite{Bhatnagar13SR}; we provide a sketch below for the sake of completeness.

We first rewrite the recursion \eqref{eq:2spsa} as follows:
For $i=1,\ldots, d$
\begin{align}
 \theta^{i}_{n+1} =& \Gamma_{i}\left(\theta^{i}_n + \gamma_n \sum_{j=1}^{d} \bar M^{i,j}(\theta_n) \nabla_{j} \C(X^\theta_n) + \gamma_n \zeta_n + \chi_{n+1} - \chi_n \right), \label{eq:pi-n}
\end{align}
where 
\begin{align*}
\bar M^{i,j}(\theta) = &\Upsilon(\nabla^2 \C(X^{\theta}))^{-1}\\
 \chi_n =& \sum_{m=0}^{n-1} \gamma_m \sum_{k=1}^{d} {\bar{M}}_{i,k}(\theta_m)\Bigg(
\frac{\C(X^{\theta_m-\delta_m\Delta_m - \delta_m\widehat\Delta_m}) -
\C(X^{\theta_m+\delta_m\Delta_m + \delta_m\widehat\Delta_m})}{2\delta_m \Delta^k_m} 
 \\
 &- E\left[\frac{\C(X^{\theta_m-\delta_m\Delta_m - \delta_m\widehat\Delta_m}) -
\C(X^{\theta_m+\delta_m\Delta_m + \delta_m\widehat\Delta_m})}{2\delta_m \Delta^k_m} 
\mid {\cal F}_m\right]\Bigg) \text{ and}\\
\zeta_n = &\E\left[\left.\dfrac{\overline \C_n^{\theta_n + \delta_n (\Delta_n + \hat \Delta_n)} -\overline \C_n^{\theta_n - \delta_n (\Delta_n + \hat \Delta_n)}}{2 \delta_n \Delta_n^i}\right| \F_n \right] - \nabla_i \C(X^{\theta_n}).
\end{align*}
In lieu of Lemmas \ref{lemma:2spsa-bias}--\ref{lemma:h-est}, it is easy to conclude that $\zeta_n \rightarrow 0$ as $n\rightarrow \infty$, $\chi_n$ is a martingale difference sequence and that $\chi_{n+1} - \chi_n \rightarrow 0$ as $n\rightarrow \infty$. 
Thus, it is easy to see that \eqref{eq:pi-n} is a discretization of the ODE:
\begin{align}
\dot\theta^{i}_t = \check\Gamma_{i}\left(- \nabla \C(X^{\theta^{i}_t}) \Upsilon(\nabla^2 \C(X^{\theta_t}))^{-1} \nabla \C(X^{\theta^{i}_t}) \right).
\label{eq:n-ode}
\end{align}
Since $\C(X^{\theta})$ serves as a Lyapunov function for the ODE \eqref{eq:n-ode}, it is easy to see that the set \\$\K = \{\theta \mid
\nabla \C(X^{\theta^{i}})  \check\Gamma_{i}\left(-\Upsilon(\nabla^2 \C(X^{\theta}))^{-1} \nabla \C(X^{\theta^{i}})\right)
=0, \forall i=1,\ldots,d\}$ is an asymptotically stable attractor set for the ODE \eqref{eq:n-ode}. The claim now follows from Kushner-Clark lemma.
\end{proof}
\end{proof}



\bibliographystyle{plainnat}

\bibliography{cpt-refs}

\end{document}